\documentclass[11pt]{article}

\usepackage[margin=1in]{geometry}

\usepackage[utf8]{inputenc} %
\usepackage[T1]{fontenc}    %
\usepackage{hyperref}       %
\usepackage{url}            %
\usepackage{booktabs}       %
\usepackage{amsfonts}       %
\usepackage{nicefrac}       %
\usepackage{microtype}      %
\usepackage{xcolor}         %

\usepackage{settings}
\usepackage{enumitem}
\usepackage{caption}
\usepackage{titletoc}

\hypersetup{
  colorlinks = true,
  citecolor = blue,
}

\title{Accelerating Speculative Diffusions via Block Verification}

\author{%
  Alexander Soen\textsuperscript{$\diamond$}\footnote{Work done while the author was at Google.}
  \qquad
  Hisham Husain\textsuperscript{\dag}
  \qquad
  Valentin De Bortoli\textsuperscript{\ddag}
  \qquad
  Arnaud Doucet\textsuperscript{\ddag}
}
\date{%
  \small
  KTH\textsuperscript{$\diamond$}
  \qquad
  Google Research\textsuperscript{\dag}
  \qquad
  Google DeepMind\textsuperscript{\ddag}
}

\begin{document}

\maketitle

\begin{abstract}
  Speculative decoding speeds up LLM inference by using a draft model to generate tokens, with an acceptance-rejection scheme that ensures that the output matches the target distribution. Adapting this to continuous diffusions is difficult because speculative sampling requires drawing from a residual distribution. While straightforward in discrete spaces, efficiently sampling this residual in continuous space is non-trivial. Consequently, existing diffusion adaptations either use computationally inefficient sampling techniques or rely on an alternative scheme. In this work, we introduce a novel scheme that efficiently implements the original speculative sampling mechanism for diffusion models. Our approach offers a critical advantage over current methods: it enables us to adapt block verification from LLMs to diffusions---which provably improves the acceptance rate of drafts. Furthermore, we formalize and analyze the Free Drafter, a heuristic self-speculative drafter for diffusions that requires no training. By enabling block verification, our Free Drafter yields up to a 6.3\% speedup over existing speculative methods with no additional training and negligible overhead beyond the existing parallel verification pass.
\end{abstract}

\section{Introduction}
Diffusion models have become ubiquitous in generative modeling~\citep{sohl2015deep,ho2020denoising,song2020score}. These models excel at generating high-fidelity samples by reversing a noising process that transforms a data distribution into Gaussian noise. However, sampling from a diffusion model requires many evaluations of large neural networks to generate data. %

Among the many approaches proposed for acceleration, one promising direction is speculative diffusions~\citep{wang2024continuous,bortoli2025accelerated,hu2025diffusion,li2025ts,wen2025free}—the application of speculative sampling to diffusion models. Originally introduced for LLMs, the core idea of speculative decoding is to leverage a small, faster drafter model to propose a sequence of candidate tokens. These tokens are then verified by the larger, more accurate target model. A subsequence of the proposed draft is accepted and taken as part of the output, and the first rejected draft token is replaced by a draw from a residual distribution~\citep{leviathan2023fast,chen2023accelerating}. A critical property of speculative sampling is that the output distribution of the algorithm exactly matches the target model's distribution—there is no loss in sample quality. In the LLMs context, block verification is a state-of-the-art technique that further improves speculative decoding while preserving the exact target distribution. Block verification provides an optimal change in the verification strategy: instead of verifying the proposed sequence token-by-token (sample verification, \aka token verification for LLMs), the entire block is jointly verified~\citep{sun2025block}.

As first observed by \cite{sun2024spectr}, speculative sampling can be viewed as a sequential implementation of the $\Gamma$-maximal coupling \citep{lindvall2002lectures}.
Given a pair of distributions $p, q$ and a draft state $\X \sim p$, a maximal coupling can be thought of as an algorithm that outputs $\tilde{\X} \sim q$ while maximizing the probability that $\tilde{\X}=\X$.
Essentially, it transforms a draft sample into a target sample while maximizing the probability of accepting the draft sample as a target sample. However, implementing $\Gamma$-maximal coupling requires efficient sampling from a residual distribution $r_\Gamma$. While this is trivial in discrete settings, continuous settings typically rely on a standard rejection sampling procedure \citep{Jacob2022coupling,wang2024continuous,subbaraman2025accelerating,zou2025fast}. Unfortunately, this approach incurs a high computational cost \citep{Jacob2022coupling, bortoli2025accelerated}. While recent work by \citet{anari2025parallelsamplingautospeculation} leverages parallelism to improve efficiency, the resulting procedure is complex, relies on a fallback tree, and suffers from stochastic execution times.
To bypass this bottleneck, speculative diffusion methods proposed in \citet{bortoli2025accelerated,hu2025diffusion} employ the reflection maximal coupling \citep{bourabee2020coupling} instead of the $\Gamma$-maximal coupling, as the former is significantly simpler and cheaper to implement.

In this paper, we make the following contributions:
\begin{itemize}[leftmargin=10pt, itemindent=0pt, parsep=0pt]
  \item \emph{Efficient $\Gamma$-maximal coupling for diffusions}: We implement $\Gamma$-maximal coupling for diffusion models using a new, one-step algorithm to sample the residual distribution $r_\Gamma$. By requiring only a single target model evaluation, our implementation significantly simplifies existing techniques \citep{wang2024continuous,anari2025parallelsamplingautospeculation} (see \cref{algo:decomposition-res,algo:u-sample}).
  \item \emph{Block verification for continuous models}: Leveraging our sampling procedure, we adapt the LLM block verification of \cite{sun2025block} to the diffusion setting (see \cref{algo:block-verification}).
  \item \emph{Theoretical analysis of coupling limitations}: We prove that a broad class of reflection-style deterministic corrections cannot be used with block verification (see \cref{prop:impossibility}).
  \item \emph{The Free Drafter}: We introduce a heuristic self-speculative mechanism—the Free Drafter—which empirically outperforms previously considered drafting strategies in speculative diffusion.
  \item \emph{Empirical speedups}: Experimentally, our block verification scheme reduces wall-clock latency by up to $6.3\%$ versus standard speculative sampling, with no additional training or computational cost.
\end{itemize}

\paragraph{Notation}
Let $\llbracket \texttt{pred} \rrbracket$ denote Iverson brackets~\citep{knuth1992two} which evaluate to $1$ when predicate $\texttt{pred}$ is true and evaluate to $0$ when false. $\Id$ denotes the identity matrix. Random variables are denoted via sans serif font, \ie, $\X, \Y, \Z$.
When discussing speculative diffusions, we denote the drafter by $p$, the target model by $q$, and the residual function by $r$. The draft size will be denoted by $\gamma$. Draft samples are denoted with a hat $\hat{y}$ and we denote sequences via $y_{i:j} = (y_k)_{k=i}^{j}$ for $i \leq j$.

\begin{table}[t]
  \caption{Summary of the speculative diffusions variants explored. Verification and residual calculation in \cref{algo:abstract-spec-diffusion} can be swapped to give the following approaches.
    $\dagger$ was proposed in \citet{bortoli2025accelerated} and \citet{hu2025diffusion}.
  }
  \label{tab:specs}
  \centering
  \resizebox{\textwidth}{!}{
    \begin{tabular}{llll}
      \toprule
      Approach        & Verification (\algolref{algo-line:verification}) & Residual (\algolref{algo-line:residual}) & Stochastic  \\
      \midrule
      Reflection$\dagger$ (\reflection) & Sample (\cref{algo:token-verification})    & Reflection (\cref{algo:reflect-res})  & \ding{55}  \\
      Decomposition (\decomposition) & Sample (\cref{algo:token-verification})    & Decomposition (\cref{algo:decomposition-res}, $\alpha = 1$)  & \ding{51}  \\
      Block (\block) & Block (\cref{algo:block-verification})    & Decomposition (\cref{algo:decomposition-res})  & \ding{51} \\
      \bottomrule
    \end{tabular}
  }
\end{table}

\section{Speculative Diffusions}

\paragraph{Diffusion models.}
Let $q_{\textup{data}}$ be a data distribution on $\mathbb{R}^d$. As in \citet{song2020score}, the forward noising process $(\X_t)_{t \in [0, 1]}$ of the model is defined by
\begin{equation}
  \label{eq:forward-process}
  \dd \X_t = f_t \X_t \, \dd t + g_t \, \dd \B_t, \qquad \X_0 \sim q_{\textup{data}},
\end{equation}
where $ (\B_t)_{t \in [0,1]} $ corresponds to $d$-dimensional Brownian motion. Let $q_t$ denote the distribution of $\X_t$ and set $f_t,g_t$ such that $q_1=\mathcal{N}(0,\Id)$.
The generative process $(\Y_t)_{t \in [0, 1]}$ is defined by
\begin{equation}
  \label{eq:backward-process}
  \dd \Y_t = b_{t}(\Y_t) \, \dd t + \varepsilon g_{1-t} \, \dd \W_t, \qquad \Y_0 \sim q_1,
\end{equation}
where $b_{t}(x) = - f_{1-t} x + \frac{1 + \varepsilon^2}{2} g_{1-t}^2 s_{1-t}(x)$ with $s_{t}(x) = \grad \log q_{t}(x)$ denoting the Stein score; and $(\W_t)_{t \in [0, 1]}$ corresponds to another $d$-dimensional Brownian motion. The churn parameter $\varepsilon \geq 0$~\citep{karras2022elucidating} controls the stochasticity of $(\Y_t)_{t \in [0, 1]}$~\citep{albergo2023stochastic}. %
One can show that $\Y_{1 - t} \sim q_t$, hence in particular $\Y_1 \sim q_{\textup{data}}$. We always set here $\varepsilon>0$.

In practice, $b_t$ is approximated via a neural network to obtain $\hat{b}_t$. At inference, a timestep discretized version of $(\Y_t)_{t \in [0,1]}$ is evaluated with $q_1= \mathcal{N}(0, \Id)$.
Denote a $K + 1$ step discretization of $[0, 1]$ as $(t_{k})_{k=0}^{K}$, with $t_k = k \delta$ and $ \delta = 1 / K$, and denote the corresponding generative process as DDPM.

\begin{algorithm}[t]
  \caption{Abstract Speculative Diffusions}\label{algo:abstract-spec-diffusion}
  \begin{algorithmic}[1]
    \Require Target model $q$, drafter model $p$, draft size $\gamma$, number of sampling steps $K$.
    \State Sample $\Y_0 \sim \mathcal{N}(0, \Id)$ and set $k = 0$.
    \While{$k < K$}
    \State Set $ {\gamma}_k = \min\{ {\gamma}, K - k \}$.
    \LineComment{Construct Draft.}
    \State Initialize draft $\hat{\Y}_{k} = \Y_k$.
    \For{$j \in \{1, \ldots, {\gamma_k}\}$}
    \State Sample $\hat{\Y}_{k+j} \sim p(\cdot \mid \hat{\Y}_{k:k+j-1})$.
    \Comment{Cache draft mean $m_{k+j-1}^p$ computed in sampling.}
    \EndFor
    \LineComment{Verify Draft.}\label{algo-line:verif-comment}
    \State Compute $m^q_{k+j-1} = m_{k+j-1}^q(\hat{\Y}_{k+j-1})$ for all $ j \in \{ 1, \ldots, \min\{ \gamma_k + 1, K - k \} \}$ in parallel. \label{algo-line:parallel-verification-primary}
    \State Compute $\tau, \Delta, \Z, m^q, \sigma, \alpha_{\textrm{res}} = \Verification( (\hat{\Y}_{k+j}, m^p_{k+j-1}, m^q_{k+j-1}, \sigma_{k+j-1})_{j=1}^{\gamma_k} )$.\label{algo-line:verification}
    \LineComment{Residual Calculation.}
    \If{$\tau = {\gamma_k} + 1$ and $k + \tau \leq K$}
    \Comment{Latter condition handles end-of-sampling case.}
    \State Sample $\Y_{k + {\gamma_k} + 1} \sim q(\cdot \mid \hat{\Y}_{k+{\gamma_k}})$.
    \ElsIf{ $k + \tau \leq K$ }
    \State Compute $\Y_{k+\tau} = \Residual(\Delta, \Z, m^q, \sigma; \alpha_{\textrm{res}})$.\label{algo-line:residual} \Comment{Ensures $\Y_{k+\tau} \sim q(\cdot \mid \hat{\Y}_{k+\tau-1})$.}
    \EndIf
    \State Set $\Y_{k+j} = \hat{\Y}_{k+j}$ for $j \in \{ 1, \ldots, \tau - 1 \}$; and set $k = \min\{ k + \tau, K \}$.
    \EndWhile
    \State \Return $\Y_{0:K}$.
  \end{algorithmic}
\end{algorithm}

\paragraph{Drafter and Target Models.}
When using an Euler--Maruyama discretization of the generative process defined by time discretization $(t_k)_{k=0}^K$, the resulting Markov chain has transition densities
\begin{align}
  q(y_{k+1} \mid y_{k}) = \mathcal{N}(y_{k+1}; m^q_{k}(y_{k}), \sigma_{k}^2 \Id),\qquad{\text{for}}\quad m^q_k(y) = y + \delta \hat{b}^q_{t_k}(y),~\sigma_k = \sqrt{\delta} \varepsilon g_{1-t_k}. \label{eq:target-distribution}
\end{align}
Speculative diffusions use a drafter model defined by transition densities
\begin{align}
  p(y_{k+1} \mid y_{k}) = \mathcal{N}(y_{k+1}; m^p_{k}(y_{j:k}), \sigma_{k}^2 \Id)\label{eq:draft-distribution},
\end{align}
where the mean is calculated with $\max\{ 0, k - \gamma + 1\} \le j < k$ and $\gamma$ is the draft length we will consider.
The drafter is selected such that $b^p_{t_k} \approx b^q_{t_k}$ with $b^p_{t_k}$ being cheaper to evaluate than $b^q_{t_k}$.
When the index $k$ is clear from context, we drop the subscript to simplify notation.
In the sequel, we assume that $m^p \neq m^q$ as otherwise we could just sample from the drafter with no quality drop.

\paragraph{Speculative Sampling.}
Speculative sampling uses the fast draft model $p$ to generate a sequence of length $\gamma$, which the target model $q$ verifies in parallel. Ideally, this verification adds minimal wall-clock latency over a single target model evaluation. \Cref{algo:abstract-spec-diffusion} outlines speculative diffusions: a drafted sequence is validated and corrected via \Verification and \Residual algorithms. We refer to one iteration of this process as a \emph{round}.
The standard instantiations of speculative sampling for autoregressive LLMs use a sample verification algorithm~\citep{leviathan2023fast,chen2023accelerating}, where its diffusion equivalent procedure \TokenVerification, is depicted in \cref{algo:token-verification}.
Alongside the sample verification procedure, the corresponding LLM \Residual algorithm involves sampling from the following distribution (\cref{algo:gamma-res}):
\begin{equation}
  \label{eq:llm-res}
  r_{\Gamma}(y) \propto \max \left\{ 0, q(y \mid \hat{y}_{k+\tau-1}) - p(y \mid \hat{y}_{k:k+\tau-1}) \right\},
\end{equation}
where $\tau$ corresponds to the (relative) index which we need to correct via the residual distribution. For sample verification, $\tau$ is the index of the first rejected draft sample (as per \cref{algo-line:token-coin-toss}, \cref{algo:token-verification}).

In the context of LLMs,
implementing \cref{algo:gamma-res} is straightforward
because the state space is finite.
However, for diffusion models on continuous spaces, sampling is significantly more complex.
Previous methods to address this have relied on standard rejection sampling \citep{Jacob2022coupling,subbaraman2025accelerating,wang2024continuous,zou2025fast}, which is computationally inefficient, exhibits random runtimes and multiple target model evaluations; see \citep{Jacob2022coupling} and \citep[Section 3.2]{bortoli2025accelerated}.
While a more sophisticated rejection variant has recently been developed \citep{anari2025parallelsamplingautospeculation}, it is complex and still suffers from random execution times and multiple target model calls.
To bypass these limitations, \citet{bortoli2025accelerated,hu2025diffusion} do not sample from $r_{\Gamma}$. Instead, they utilize a deterministic correction provided by the reflection coupling of diffusion models \citep{bourabee2020coupling}, as per \cref{algo:reflect-res}---their proposed speculative diffusion procedures can be summarized as an instantiation of speculative sampling with \TokenVerification and \Reflection.

Although an efficient algorithm utilizing $r_{\Gamma}$ has remained elusive, it can be shown that using either the residual sampling
(\cref{algo:gamma-res})
or the output of a reflection coupling (\cref{algo:reflect-res}) guarantees that the final output of \cref{algo:abstract-spec-diffusion} follows the target distribution $q$ \citep{leviathan2023fast,bortoli2025accelerated,hu2025diffusion,yin2024theoretical}.

\begin{algorithm}[t]
  \caption{$\TokenVerification( (\hat{y}_{k+i}, m^p_{k+i-1}, m^q_{k+i-1}, \sigma_{k+i-1})_{i=1}^\gamma )$}\label{algo:token-verification}
  \begin{algorithmic}[1]
    \Require Draft $\hat{y}_{k+1:k+\gamma}$, draft means $m_{k:k+\gamma-1}^p$, target means $m_{k:k+\gamma-1}^q$, variances $\sigma^2_{k:k+\gamma-1}$.
    \State Set $\bm{c}_0 = 1$.
    \For{$j \in \{ 1, \ldots, \gamma \}$}
    \State Set $\Delta_{j} = (m_{k+j-1}^p - m_{k+j-1}^q) / \sigma_{k+j-1}$ and $\Z_{j} = (\hat{y}_{k+j} - m_{k+j-1}^p) / \sigma_{k+j-1}$.
    \State Calculate $\alpha_j = \min\left\{1, \frac{\mathcal{N}(\Z_{j} + \Delta_{j}; 0, \Id)}{\mathcal{N}(\Z_{j}; 0, \Id)} \right\}$.\label{algo-line:token-coin}
    \State Flip coin $\bm{c}_j = \llbracket \eta < \alpha_j \rrbracket$, where $\eta \sim \Unif[0, 1]$.\label{algo-line:token-coin-toss}
    \EndFor
    \State Set $\tau = \min\left(\left\{ j \in \{ 1, \ldots, \gamma \} \mid \bm{c}_j = 0 \right\} \cup \{ \gamma + 1 \}\right)$.
    \State \Return $\tau$, $\Delta = \Delta_{\tau}$, $\Z = \Z_{\tau}$, $m^q = m^q_{k+\tau-1}$, $\sigma = \sigma_{k+\tau-1}$, $\alpha_{\textrm{res}} = 1$.%
    \footnotemark
  \end{algorithmic}
\end{algorithm}

\begin{figure}[t]
  \hrule height .8pt depth 0pt
  \kern 2pt
  \begin{minipage}[t]{0.48\textwidth}%
    \captionof{algorithm}{$\Gamma\texttt{-Coupling}(p, q)$}%
    \label{algo:gamma-res}%
    \kern 2pt
    \hrule
    \kern 2pt
  \end{minipage}
  \hfill
  \begin{minipage}[t]{0.48\textwidth}
    \captionof{algorithm}{$\Reflection(\Delta, \Z, m^q, \sigma)$}%
    \label{algo:reflect-res}%
    \kern 2pt
    \hrule
    \kern 2pt
  \end{minipage}
  \par
  \noindent\textbf{Require:} Denoising difference $\Delta$, draft noise $\Z$,
  target mean $m^q$, variance $\sigma^2 > 0$ (ignore $\alpha_{\textrm{res}}$). %
  \par\vspace{0.2em}
  \begin{minipage}[b]{0.48\textwidth}
    \begin{algorithmic}[1]
      \State Define residual distribution:
      \Statex \qquad $r_{\Gamma}(y) \propto \max\{ 0, q(y) - p(y) \}$.
      \State \Return $\Y \sim r_\Gamma$. %
    \end{algorithmic}
  \end{minipage}
  \hfill
  \begin{minipage}[b]{0.48\textwidth}
    \begin{algorithmic}[1]
      \State Compute normalized $e = \Delta / \Vert \Delta \Vert$.
      \State Compute noise $\Z_\mathrm{r} = (\Id - 2ee^\top)\Z$.
      \State \Return $\Y = m^q + \sigma \Z_\mathrm{r}$.
    \end{algorithmic}
  \end{minipage}
  \par\vspace{0.2em}
  \kern2pt\hrule\relax
\end{figure}

\section{Alternative Residuals and Block Verification}

In this section, we first introduce a method to sample from a generalized version of the residual distribution $r_{\Gamma}$ for diffusion models. Second, we leverage this sampling procedure to derive a block verification framework for diffusions. Finally, we introduce a self-speculative drafter and analyze its impact on inference acceleration.

\footnotetext{When $\tau=\gamma + 1$, the values $\Delta$, $\Z$, $m^q$, and $\sigma$ can be arbitrary as they will not be used in any subsequent calculation (see the `if-condition' of the residual calculation in \cref{algo:abstract-spec-diffusion}). A more precise pseudo-code could combine the verification and residual calculation procedures to avoid these undefined variables, see for instance \cref{algo:full-block-verification}. \label{foot:undefined}}

\subsection{Sampling the Residual Distribution for Diffusions}
We show that sampling from the residual distribution $r_\Gamma$ can be achieved in deterministic time and a single target model evaluation, drastically simplifying existing methods.
Sampling from \cref{eq:llm-res} for diffusions boils down to sampling from $r(y) \propto \max\{ 0, q(y) - p(y)\}$, where $p = \mathcal{N}(m^p, \sigma^2 \Id)$ and $q = \mathcal{N}(m^q, \sigma^2 \Id)$.
We propose a method leveraging an orthogonal decomposition of Gaussian random variables: %
\ding{172} a 1D sampling task corresponding to the mean difference projection; %
and \ding{173} a Gaussian component. More precisely, we define:
\begin{equation}
  \Delta \defeq (m^p - m^q) / \sigma
  \quad \textrm{and} \quad
  e \defeq \Delta / \Vert \Delta \Vert.
\end{equation}
The vector $\Delta$ represents the discrepancy in denoising directions between the draft and target models, while $e$ is its normalized counterpart. We assume throughout that $\Vert \Delta \Vert \neq 0$. Note that these quantities are also central to the definition of the reflection coupling residual in \cref{algo:reflect-res}.

\begin{proposition}%
  \label{prop:decomposition}
  Let $p = \mathcal{N}(m^p, \sigma^2 \Id)$, $q = \mathcal{N}(m^q, \sigma^2 \Id)$, $\alpha \in (0, 1]$, and $\Vert \Delta \Vert \neq 0$. The distribution $r(y; \alpha) \propto \max\{ 0, \alpha q(y) - p(y) \}$ can be sampled as follows:
  \begin{equation}%
    \label{eq:decomp}
    \U \sim \psi_{\alpha};
    \qquad
    \Z_\perp \sim \mathcal{N}(0, \Id - ee^\top);
    \qquad
    \Y = m^q + \sigma (\U e + \Z_\perp),
  \end{equation}
  where $ \psi_{\alpha}(u) \propto \max\{ 0, \alpha \mathcal{N}(u; 0, 1) - \mathcal{N}(u - \Vert \Delta \Vert; 0, 1) \} $ is a univariate distribution; with \cdf
  \begin{equation}
    \label{eq:u-cdf}
    \Psi_{\alpha}(u)
    =
    \frac{\alpha \Phi(u) - \Phi(u - \Vert \Delta \Vert)}{\alpha \Phi \left(\frac{\log \alpha}{\Vert \Delta \Vert} + \frac{\Vert \Delta \Vert}{2}\right) - \Phi \left(\frac{\log \alpha}{\Vert \Delta \Vert} - \frac{\Vert \Delta \Vert}{2}\right)}
    \quad
    \textrm{if}
    \quad
    u < \frac{\log \alpha}{\Vert \Delta \Vert} + \frac{\Vert \Delta \Vert}{2}
  \end{equation}
  and $\Psi_{\alpha}(u)=1$ otherwise, with $\Phi$ being the \cdf of the standard normal distribution.
\end{proposition}

The parameter $\alpha$ in \cref{prop:decomposition} is essential for the block verification framework discussed in the sequel. For standard speculative decoding, setting $\alpha = 1$ recovers the original residual distribution $r_\Gamma$. This proposition effectively reduces a high-dimensional sampling problem to a tractable 1D task. Once $\U \sim \psi_{\alpha}$ is obtained, the subsequent sampling of $\Z_\perp$ and $\Y$ in \cref{eq:decomp} is straightforward to implement.
Sampling $\U$ is also simple as its \cdf $\Psi_\alpha$ can be easily computed pointwise as per \cref{eq:u-cdf} (e.g., using  the standard normal \cdf $\Phi$), so the inverse sampling method can be used. In practice, we employ a bisection search %
(see \cref{app:u_sampling} for details).
The complete sampling procedure is summarized in \cref{algo:decomposition-res}.
By replacing the reflection coupling in %
\cref{algo:abstract-spec-diffusion} (at \algolref{algo-line:residual}) with
$\Decomposition(\Delta, m^q, \sigma; \alpha = 1)$, we recover the standard speculative decoding framework \citep{leviathan2023fast} but adapted to diffusion models.

\begin{algorithm}
  \caption{$\Decomposition(\Delta, m^q, \sigma; \alpha)$}%
  \label{algo:decomposition-res}%
  \begin{algorithmic}[1]
    \Require Denoising difference $\Delta$, target mean $m^q$, variance $\sigma^2 > 0$, weight $\alpha \in (0, 1]$ (ignore $\Z$).
    \State Sample $\U \sim \psi_{\alpha}$ via \cref{eq:u-cdf}
    \State Sample $\G \sim \mathcal{N}(0,\Id)$, set $\Z_\perp = \G - e(e^\top \G)$, where $e = \Delta / \Vert \Delta \Vert $ so that $\Z_\perp \sim \mathcal{N}(0, \Id - ee^\top)$.
    \State \Return $\Y = m^q + \sigma (\U e + \Z_\perp)$.
  \end{algorithmic}
\end{algorithm}

\subsection{Block Verification}
Standard speculative decoding is derived by maximizing $\min\{ 1, q(\hat{y})/p(\hat{y}) \}$, the \emph{per-sample} acceptance rate, which corresponds to the probabilities $\alpha$ calculated in \TokenVerification (\cref{algo:token-verification}, \algolref{algo-line:token-coin}). One way to improve speculative decoding is to alter the verification function (\cref{algo:abstract-spec-diffusion}, \algolref{algo-line:verification}). Block verification provides such an alternative: rather than maximizing the acceptance rate for each sample independently, it maximizes the joint acceptance rate of the entire draft block $\hat{y}_{k+1:k+\gamma}$ which reduces the number of target model evaluations \citep{sun2025block}.

We adapt block verification to diffusion models by replacing \TokenVerification with the procedure \BlockVerification, as defined in \cref{algo:block-verification}. To ensure that we sample exactly from the target distribution $q$, one must employ a generalized variant of the residual distribution $r_{\Gamma}$:
\begin{equation}
  \label{eq:block-res}
  r_\textrm{BLOCK}(y) \propto \max \left\{ 0, \alpha_{\tau-1} q(y \mid \hat{y}_{k+\tau-1}) - p(y \mid \hat{y}_{k:k+\tau-1}) \right\},
\end{equation}
where $\alpha_{\tau-1}$ is defined recursively in \cref{algo:block-verification}.
Note that in block verification $\tau$ corresponds to the index after the last accepted sample (as per \cref{algo-line:block-coin-toss}, \cref{algo:block-verification}).
A derivation of the acceptance probabilities $h_j$ used in \cref{algo:block-verification} is provided in the \cref{sec:block-verification-proof}.

Our previous derivation of the decomposition residual for the sample-wise case directly enables sampling from \cref{eq:block-res}. Thus, $r_\textrm{BLOCK}$ can be sampled by applying \cref{prop:decomposition} and \cref{algo:decomposition-res} with $\alpha = \alpha_{\tau-1}$. Notably, because \Decomposition samples directly from $r_\textup{BLOCK}$,
instantiating \cref{algo:abstract-spec-diffusion} with \BlockVerification and \Decomposition yields a block verification algorithm whose output is consistent with the target model $q$—thereby recovering the theoretical guarantees of \citet[Theorem 1]{sun2025block}.

\begin{algorithm}
  \caption{$\BlockVerification( (\hat{y}_{k+i}, m^p_{k+i-1}, m^q_{k+i-1}, \sigma_{k+i-1})_{i=1}^\gamma )$}\label{algo:block-verification}
  \begin{algorithmic}[1]
    \Require Draft $\hat{y}_{k+1:k+\gamma}$, draft means $m_{k:k+\gamma-1}^p$, target means $m_{k:k+\gamma-1}^q$, variances $\sigma^2_{k:k+\gamma-1}$.
    \State Initialize $\bm{c}_0 = 1$ and $\alpha_0 = 1$.
    \For{$j \in \{ 1, \ldots, \gamma \}$}
    \State Set $\Delta_{j} = (m_{k+j-1}^p - m_{k+j-1}^q) / \sigma_{k+j-1}$ and $\Z_j = (\hat{y}_{k+j} - m_{k+j-1}^p) / \sigma_{k+j-1}$.
    \State Calculate $\alpha_j = \min\left\{1, \alpha_{j-1} \frac{\mathcal{N}(\Z_{j} + \Delta_{j}; 0, \Id)}{\mathcal{N}(\Z_j; 0, \Id)} \right\}$.\label{algo-line:block-coin}
    \EndFor
    \For{$j \in \{ 1, \ldots, \gamma \}$}
    \If{$j \neq \gamma$}
    \State Calculate $h_j = {v_j}/({v_j + 1 - \alpha_j})$, where
    \Statex \qquad \qquad \qquad $
    v_j = \alpha_j \Phi\left( \frac{\log \alpha_j}{\Vert \Delta_{j+1} \Vert} + \frac{\Vert \Delta_{j+1} \Vert}{2}\right) - \Phi\left( \frac{\log \alpha_j}{\Vert \Delta_{j+1} \Vert} - \frac{\Vert \Delta_{j+1} \Vert}{2}\right).
    $
    \Else
    \State Set $h_\gamma = \alpha_\gamma$.
    \EndIf
    \State Flip coin $\bm{c}_j = \llbracket \eta < h_j \rrbracket$, where $\eta \sim \Unif[0, 1]$.\label{algo-line:block-coin-toss}
    \EndFor
    \State Set $\tau = 1 + \max\left\{ j \in \{ 0, 1, \ldots, \gamma\} \mid \bm{c}_j = 1 \right\}$.
    \State \Return $\tau$, $\Delta = \Delta_{\tau}$, $\Z = \Z_{\tau}$, $m^q = m^q_{k+\tau-1}$, $\sigma = \sigma_{k+\tau-1}$, $\alpha_{\textrm{res}} = \alpha_{\tau - 1}$.\footref{foot:undefined}
  \end{algorithmic}
\end{algorithm}

A natural question is whether a deterministic residual calculation, such as the one used by reflection coupling, exists for the block verification case. The argument below rules out a broad class of simple reflection-style corrections, which motivates the stochastic residual used in block verification.

\begin{proposition}[Informal]\label{prop:impossibility}
  For draft sequences of length $\gamma \geq 2$, there exists no valid deterministic residual correction of the form $t(\Z)$ for block verification that is simultaneously invertible and admits a constant Jacobian factor under change of variables.
\end{proposition}

The result rules out deterministic affine-style corrections on the noise, which the \Reflection algorithm for sample verification relies on.
This result motivates the need for the original speculative sampling $\Gamma$-coupling to unlock the efficiency gains of parallel block validation.

\subsection{Self-Speculative Drafter}
In the context of speculative diffusion, self-speculative approaches have proven highly effective \citep{bortoli2025accelerated}. One such strategy involves extrapolating a single evaluation of the target model $q$ to generate a draft sequence over $i \in \{1, \ldots, \gamma\}$ steps, a method we refer to as the \emph{Frozen Drafter}. This approach is motivated by the exchangeability property of Ornstein--Uhlenbeck diffusion models \citep{hu2025diffusion}, which can be derived via stochastic localization \citep{eldan2013thin, el2022information}. Given a time discretization $(t_{k})_{k=0}^{K}$ and an initial denoising state $y_{k}$, the drafting strategy is defined by setting the drift terms $b_{t_{k+i-1}}^p$ for $i \in \{1, \dots, \gamma\}$ as:
\begin{equation}
  \label{eq:frozen-draft-drift}
  b_{t_{k+i-1}}^p(\hat{y}) = -f_{1-t_{k+i-1}} \hat{y} + \frac{1+\varepsilon^2}{2} g^2_{1-t_{k+i-1}} s^q_{1-t_{k}}(y_{k}),
\end{equation}
where $s^q_{\cdot}$ denotes the target model's neural network approximation of the score function $s$.
Crucially, this expression differs from the standard drift $b^q_{t_{k+i-1}}(y)$ by utilizing a time-delayed score $s^q_{1-t_k}(y_k)$, evaluated only once at the beginning of the current speculative round. Consequently, for a drafting window of size $\gamma$, we require only a single call to the target model to generate the entire draft sequence by sampling from the resulting distribution.

An appealing property of the Frozen Drafter \cref{eq:frozen-draft-drift} is that it provides a theoretical guarantee on the number of (parallel) target model calls $q$ in the sample verification setting. Specifically, \citet{hu2025diffusion} established that, under mild data assumptions and an appropriate choice of $\gamma$, the number of parallel calls scales asymptotically as $\bigoh(K^{2/3} (\beta d \delta)^{1/3})$, where $\beta$ depends on the data assumption. We adapt this result here to speculative diffusions with block verification as follows.

\begin{proposition}%
  \label{thm:complexity}
  Assume the use of the Frozen Drafter and $\Tr(\cov)[q_{\textup{data}}] \leq \beta d$.
  For a fixed draft size $\gamma$, let $\rho(\gamma) \in [0,1]$ denote the ratio of the expected number of rounds required to complete speculative diffusion with block verification relative to sample verification.
  Taking $\gamma \asymp (\nicefrac{K}{\beta \delta d})^{1/3}$, the expected number of parallel target model calls under block verification is at most $\bigoh({\rho}(\gamma) K^{2/3} (\beta d \delta)^{1/3})$.
\end{proposition}

The existence of ${\rho}(\gamma) \in [0, 1]$ is guaranteed from the optimality of block verification over alternative sample verification strategies~\citep[Theorem 2]{sun2025block}.

While the Frozen Drafter provides theoretical guarantees on the number of parallel target model calls (\cref{thm:complexity}), it admits a significant practical limitation. As per \cref{eq:frozen-draft-drift}, each round of speculative diffusion requires computing $s^q_{1-t_{k}}(y_{k})$. Consequently, the overhead of generating the draft sequence $\hat{y}$ is approximately equal to the complexity of a single target model call $q$. Hence, even in ideal settings, any speculative diffusion using the Frozen Drafter will achieve limited speedups.
This limitation is tied to the \emph{block efficiency} of a speculative diffusion algorithm: the number of denoising steps computed after a round of speculative diffusion.

\begin{proposition}[Informal]%
  \label{prop:speedup}
  Let $\be(k)$ denote the block efficiency of a speculative diffusion algorithm with the Frozen Drafter at the $k^{\text{th}}$ denoising step. The speedup compared to no speculative diffusion in ideal conditions is approximately equal to $\be(k) / 2$.
\end{proposition}
Here, ideal conditions assume that the wall-clock time of a single target model evaluation dominates all other costs (\eg, drafter overhead, residual calculations, \etc) and that the parallel execution of the target model includes negligible overhead. In practice, with these additional costs, the speedup will be worse than $\be(k) / 2$.
This suggests that if the Frozen Drafter exhibits low block efficiency or employs a small draft length $\gamma$ (noting that $\be(k) \in [1, \gamma + 1]$), the use of the Frozen Drafter potentially results in a slow-down compared to vanilla sampling.

To circumvent the wall-time dependence on the target model $q$ inherent in the Frozen Drafter, \citet[Appendix B]{bortoli2025accelerated} proposed a variant that extracts the target model's score function from the verification step of the previous speculative round. Specifically, rather than initiating the current drafting round $k$ by calling the target model to compute $s^q_{1-t_{k}}(y_{k})$, we reuse the previously computed score $s^q_{1-t_{k}}(\hat{y}_{k})$ obtained during the parallel verification phase (\algolref{algo-line:parallel-verification-primary} in \cref{algo:abstract-spec-diffusion}).
In the case in which the previous speculative round resulted in a full accept of the draft, we reuse the score utilized to compute the extra sample $\Y_{k+\tau}$ (also obtained during the parallel verification on \algolref{algo-line:parallel-verification-primary} in \cref{algo:abstract-spec-diffusion}).

In this scheme, the target model is only explicitly called at the very beginning of the first speculative round, where no prior verification results are available. We refer to this approach as the \emph{Free Drafter}, as the marginal cost of drafting becomes negligible compared to a full target model evaluation. Given this substantial reduction in overhead, a straightforward corollary of \cref{prop:speedup} follows.

\begin{corollary}[Informal]
  Under ideal conditions, the Frozen Drafter is only faster than the Free Drafter if its block efficiency is approximately twice as large.
\end{corollary}
Ideal conditions correspond to the case where the cost (in wall-clock time) of parallelizing the target model $\gamma + 1$ times is similar to a single sequential call of the target model, the cost of Free Drafter is negligible, and the cost of residual calculation is negligible.
A formal statement and proof of this result are provided in \cref{app:drafters}. As block efficiency is at most $\gamma + 1$, this corollary implies that if the Free Drafter maintains a sufficiently high block efficiency, it will strictly outperform the Frozen Drafter in wall-clock time. This remains true despite the theoretical expectation that the Frozen Drafter might achieve higher block efficiency due to its closer alignment with the exchangeability properties of diffusion models \citep{hu2025diffusion}.

\section{Related Work}

\paragraph{Accelerating Diffusions.}
Besides speculative diffusions, various alternative approaches for accelerating diffusion model sampling exist.
One line of work aims to distill a teacher diffusion model into a student which maintains high quality samples with a reduced number of denoising steps~\citep{luhman2021knowledge,salimans2022progressive,berthelot2023tract,liu2023instaflow,meng2023distillation,sauer2023adversarial,song2023consistency,katzir2023noise,kim2023consistency,xu2024ufogen,yin2024one,xu2025one,boffi2025build}.
Importantly, distillation requires the explicit training of a separate student model and typically has a drop in sample quality when compared to the teacher model~\citep{luo2023comprehensive,dieleman2024distillation}.
Other approaches propose modifications to the sampling procedure of diffusion models through improved integrators~\citep{dockhorn2022genie, liu2022pseudo,lu2022dpm,xiao2021tackling,zhangfast2023}.
Similar to speculative diffusions, other approaches leverage parallel sampling via Picard iteration~\citep{pokle2022deep,shih2023parallel,chen2024accelerating,li2024distrifusion,ma2024deepcache,tang2024accelerating}. Unlike the usual diffusion model sampling (and speculative diffusions), Picard iteration relies on repeated parallel sampling calls on a window until a fixed point is detected. It should however be noted that many of these approaches can be used in conjunction with speculative diffusions~\citep{bortoli2025accelerated}.

\paragraph{Speculative Decoding for Diffusions.}
\citet{bortoli2025accelerated,hu2025diffusion} proposed speculative diffusions concurrently by leveraging the reflection coupling residual and the Frozen Drafter used for their self-speculative approaches. \citet{hu2025diffusion} uncovers a hidden exchangeability property via stochastic localization~\citep{eldan2013thin,chen2022localization}, which resulted in a speedup guarantee over regular diffusion model sampling. \citet{bortoli2025accelerated} outlined the heuristic Free Drafter, but we provide a further analysis on why the Frozen Drafter is limited in terms of its possible speedups.
In the context of Masked Autoregressive models~\citep{li2024autoregressive}, a previous attempt at sampling from the original LLM residual in a continuous sampling space was explored~\citep{wang2024continuous,subbaraman2025accelerating,zou2025fast}. However, this approach relies on rejection sampling with the primary model as its proposal, which ends up being computationally inefficient~\citep{Jacob2022coupling,bortoli2025accelerated}.
Speculative decoding has also been successful in speeding up Jacobi decoding~\citep{ortega2000iterative,song2021accelerating} for autoregressive image generation~\citep{teng2025accelerating}.
For text, speculative decoding has been used to speed up discrete diffusions~\citep{agrawal2025spiffy,gao2025self,campbell2025self} and diffusion models have been used as drafters for LLM speculative decoding~\citep{li2025diffuspec,christopher2025speculative}.

\begin{table}[t]
  \caption{Wall-clock speedups over DDPM and block efficiency over all datasets for $K = 250$ denoising steps and drafting size $\gamma = 7$. Free Drafter is used in all instances.
    Non-FID values are calculated over 500 samples with $\pm$ error ranges approximated via error propagation.
  }
  \label{tab:summary-250-7}
  \centering
  \small
  \begin{tabular}{ccccccccc}
    \toprule
    & & \multicolumn{4}{c}{Wall-clock Speedup over DDPM} & \multicolumn{3}{c}{Avg. Block Efficiency} \\
    \cmidrule(lr){3-6}
    \cmidrule(lr){7-9}
    Dataset & $\varepsilon$ & \reflection & \decomposition & \block & \reflection$\uparrow$\block \% & \reflection & \decomposition & \block \\ %
    \midrule
    \tabledataset{CelebA}{LDM}      & 0.25 & $2.20 \pm 0.09$ & $2.19 \pm 0.09$ & $2.25 \pm 0.10$ & $2.10\%$ & $3.89$ & $3.86$ & $4.19$ \\ %
    & 0.50 & $2.15 \pm 0.09$ & $2.14 \pm 0.09$ & $2.20 \pm 0.10$ & $2.43\%$ & $3.90$ & $3.90$ & $4.20$ \\ %
    & 0.75 & $2.01 \pm 0.09$ & $2.00 \pm 0.09$ & $2.06 \pm 0.09$ & $2.01\%$ & $3.68$ & $3.68$ & $3.96$ \\ %
    & 1.00 & $1.88 \pm 0.08$ & $1.87 \pm 0.08$ & $1.90 \pm 0.08$ & $1.47\%$ & $3.41$ & $3.42$ & $3.67$ \\ %
    \cmidrule{1-9}
    \tabledataset{CelebA}{Pixel}    & 0.25 & $3.20 \pm 0.20$ & $3.18 \pm 0.21$ & $3.35 \pm 0.21$ & $4.81\%$ & $4.58$ & $4.58$ & $5.00$ \\ %
    & 0.50 & $3.15 \pm 0.21$ & $3.14 \pm 0.20$ & $3.30 \pm 0.20$ & $4.83\%$ & $4.53$ & $4.53$ & $4.97$ \\ %
    & 0.75 & $2.94 \pm 0.20$ & $2.95 \pm 0.19$ & $3.11 \pm 0.20$ & $5.88\%$ & $4.22$ & $4.26$ & $4.70$ \\ %
    & 1.00 & $2.70 \pm 0.19$ & $2.71 \pm 0.19$ & $2.87 \pm 0.19$ & $6.28\%$ & $3.84$ & $3.87$ & $4.34$ \\ %
    \cmidrule{1-9}
    \tabledataset{ImageNet}{LDM}    & 0.25 & $2.24 \pm 0.10$ & $2.23 \pm 0.10$ & $2.28 \pm 0.11$ & $1.76\%$ & $4.02$ & $4.01$ & $4.25$ \\ %
    & 0.50 & $2.20 \pm 0.10$ & $2.20 \pm 0.10$ & $2.25 \pm 0.11$ & $2.32\%$ & $4.02$ & $4.02$ & $4.30$ \\ %
    & 0.75 & $2.07 \pm 0.10$ & $2.08 \pm 0.10$ & $2.12 \pm 0.11$ & $2.25\%$ & $3.85$ & $3.86$ & $4.08$ \\ %
    & 1.00 & $1.93 \pm 0.09$ & $1.92 \pm 0.09$ & $1.97 \pm 0.10$ & $2.07\%$ & $3.61$ & $3.61$ & $3.84$ \\ %
    \cmidrule{1-9}
    \tabledataset{ImageNet}{Pixel}  & 0.25 & $2.82 \pm 0.23$ & $2.81 \pm 0.22$ & $2.96 \pm 0.22$ & $4.70\%$ & $4.15$ & $4.14$ & $4.53$ \\ %
    & 0.50 & $2.81 \pm 0.25$ & $2.80 \pm 0.24$ & $2.94 \pm 0.24$ & $4.66\%$ & $4.17$ & $4.17$ & $4.56$ \\ %
    & 0.75 & $2.63 \pm 0.24$ & $2.61 \pm 0.22$ & $2.75 \pm 0.25$ & $4.62\%$ & $3.89$ & $3.89$ & $4.28$ \\ %
    & 1.00 & $2.42 \pm 0.23$ & $2.41 \pm 0.21$ & $2.56 \pm 0.23$ & $5.57\%$ & $3.53$ & $3.54$ & $3.95$ \\ %
    \cmidrule{1-9}
    \tabledataset{LSUN}{Pixel}      & 0.25 & $2.74 \pm 0.19$ & $2.74 \pm 0.17$ & $2.86 \pm 0.19$ & $4.44\%$ & $4.21$ & $4.23$ & $4.66$ \\ %
    & 0.50 & $2.72 \pm 0.19$ & $2.72 \pm 0.18$ & $2.83 \pm 0.20$ & $4.08\%$ & $4.19$ & $4.22$ & $4.63$ \\ %
    & 0.75 & $2.52 \pm 0.18$ & $2.54 \pm 0.17$ & $2.65 \pm 0.18$ & $5.22\%$ & $3.85$ & $3.90$ & $4.34$ \\ %
    & 1.00 & $2.31 \pm 0.16$ & $2.33 \pm 0.16$ & $2.42 \pm 0.17$ & $4.59\%$ & $3.49$ & $3.53$ & $3.94$ \\ %
    \cmidrule{1-9}
    \tabledataset{CIFAR10}{Pixel}   & 0.25 & $3.61 \pm 0.25$ & $3.59 \pm 0.24$ & $3.60 \pm 0.24$ & $-0.04\%$ & $6.07$ & $6.06$ & $6.35$ \\ %
    & 0.50 & $3.65 \pm 0.24$ & $3.67 \pm 0.24$ & $3.67 \pm 0.24$ & $0.54\%$ & $6.23$ & $6.25$ & $6.52$ \\ %
    & 0.75 & $3.58 \pm 0.24$ & $3.56 \pm 0.23$ & $3.59 \pm 0.22$ & $0.24\%$ & $6.14$ & $6.16$ & $6.44$ \\ %
    & 1.00 & $3.46 \pm 0.22$ & $3.45 \pm 0.23$ & $3.47 \pm 0.21$ & $0.24\%$ & $5.98$ & $5.99$ & $6.30$ \\ %
    \bottomrule
  \end{tabular}
\end{table}

\section{Experiments}
We empirically explore the improvements that block verification has over sample verification in speculative diffusions. We compare the three verification and residual combinations (\cref{tab:specs}): (\reflection) \TokenVerification with the \Reflection residual; (\decomposition) \TokenVerification with the \Decomposition residual; (\block) \BlockVerification with the \Decomposition residual.
We denote the usual Euler--Maruyama integration of the sampling process as DDPM.
For efficiency metrics, we consider wall-clock speedups (\wrt sampling without speculation) and average block efficiency---the average accepted number of samples per round of speculative diffusions, averaged over the $K$ denoising steps.
To empirically verify that quality does not change, we report the Fr\'{e}chet Inception Distance (FID)~\citep{Heusel:2017}.
We additionally show empirically that the heuristic Free Drafter provides practical speedups over the Frozen Drafter, despite its superior block efficiency.

A large collection of image datasets are considered in our experiments. In pixel space, we consider CIFAR10 ($ 32 \times 32 \times 3$), LSUN (bedroom) ($64 \times 64 \times 3$), CelebA ($64 \times 64 \times 3$), and ImageNet ($64 \times 64 \times 3$). In latent space, we consider CelebA ($256 \times 256 \times 3$) and ImageNet ($256 \times 256 \times 3$), both with latent spaces of dimension $64 \times 64 \times 3$. All models utilize a U-Net backbone architecture, with further details in \cref{app:experiments}.
We focus on ImageNet LDM when displaying our results.
All experiments are computed using Cloud TPU V3 hardware.
Following \citet{shih2023parallel,hu2025diffusion}, when making our parallel primary call in verification, we parallelize over multiple compute units ($\gamma + 1$ given a draft size $\gamma$).

\paragraph{Summary.}
\Cref{tab:summary-250-7} summarizes the various efficiency metrics for our experiments. Speculative diffusion speeds up standard DDPM by $1.9$--$3.6 \times$ in wall-clock time. While switching between \Reflection (\reflection) and \Decomposition (\decomposition) residuals has minimal impact, block verification typically improves wall-clock speedups by $1.5$--$6.3\%$ over sample verification. Additionally, block verification consistently yields superior block efficiency. The only exception is CIFAR10, where an already high block efficiency ($\sim 6$) prevents further wall-clock gains because the marginal efficiency increase cannot overcome the slightly higher overhead of the block verification scheme.

\begin{table}
  \caption{Wall-clock speedups and 50k FID values for ImageNet LDM over multiple denoising steps for churn value $\varepsilon = 0.5$ and window size $\gamma = 7$. Free Drafter is used in all instances. Non-FID values are calculated over 500 samples with $\pm$ error ranges approximated via error propagation.}
  \label{tab:imagenet-fids}
  \centering
  \small
  \begin{tabular}{lcccccccc}
    \toprule
    & \multicolumn{4}{c}{Wall-clock Speedup} & \multicolumn{4}{c}{FID} \\
    \cmidrule(lr){2-5}
    \cmidrule(lr){6-9}
    Steps & \reflection & \decomposition & \block & \reflection$\uparrow$\block\% & \reflection & \decomposition & \block & DDPM \\
    \midrule
    50 & $1.22 \pm 0.06$ & $1.22 \pm 0.06$ & $1.22 \pm 0.07$ & $0.09\%$ & $10.83$ & $10.77$ & $10.82$ & $10.78$ \\
    100 & $1.55 \pm 0.08$ & $1.55 \pm 0.08$ & $1.57 \pm 0.08$ & $1.47\%$ & $10.25$ & $10.06$ & $9.84$ & $10.28$ \\
    250 & $2.20 \pm 0.10$ & $2.20 \pm 0.10$ & $2.25 \pm 0.11$ & $2.32\%$ & $10.09$ & $9.55$ & $9.63$ & $9.81$ \\
    500 & $2.90 \pm 0.12$ & $2.91 \pm 0.13$ & $3.01 \pm 0.14$ & $3.82\%$ & $9.98$ & $9.66$ & $9.51$ & $9.64$ \\
    1000 & $3.78 \pm 0.13$ & $3.79 \pm 0.14$ & $3.91 \pm 0.14$ & $3.67\%$ & $9.90$ & $9.65$ & $9.41$ & $9.61$ \\
    \bottomrule
  \end{tabular}
\end{table}

\paragraph{Number of sampling steps.} In \cref{tab:imagenet-fids} we explore the wall-clock speedups and FID scores for ImageNet LDM for churn parameter $\varepsilon = 0.5$ and a draft size of $\gamma = 7$ over various denoising steps. From the quality side, we see that the FIDs are all similar across the different speculative diffusions approaches when compared to DDPM. There are small differences when comparing between the speculative diffusions variants, but this is within expected numerical and stochastic variation. One trend that we find is that a higher number of denoising steps typically allows for a larger wall-clock speedup (compared to DDPM). We attribute this trend to the diffusion noise schedule property that as $k \rightarrow K$, $\sigma_{k} \rightarrow 0$. In both versions of verification, as $\sigma_k$ decreases the acceptance of a draft image decreases. Intuitively, accepting a draft image becomes increasingly more difficult as we get closer to denoising a clean image. As a result, a larger number of denoising steps $K$ allows us to stay in the high noise $\sigma_k$ regime, and thus allows us to have increased speedups.

\begin{table}
  \caption{Wall-clock speedups and block efficiency values for ImageNet LDM over multiple denoising steps for churn value $\varepsilon = 1.0$ and window size $\gamma = 7$. Inside parenthesis are the measurements compared to Free Drafter as a percentage improvement. Quantities are calculated over 500 samples.}
  \label{tab:frozen-compare}
  \centering
  \resizebox{\textwidth}{!}{
    \small
    \begin{tabular}{lllllll}
      \toprule
      & \multicolumn{3}{c}{Frozen Wall-clock Speedup over DDPM} & \multicolumn{3}{c}{Frozen Average Block Efficiency} \\
      \cmidrule(lr){2-4}
      \cmidrule(lr){5-7}
      Steps & R & D & B & R & D & B \\
      \midrule
      50 & $1.07$ ($-8.0\%$) & $1.08$ ($-7.3\%$) & $1.08$ ($-7.1\%$) & $2.76$ ($40.9\%$) & $2.76$ ($44.3\%$) & $2.87$ ($39.5\%$) \\
      100 & $1.26$ ($-11.5\%$) & $1.26$ ($-10.8\%$) & $1.28$ ($-11.4\%$) & $3.39$ ($32.3\%$) & $3.40$ ($32.2\%$) & $3.54$ ($28.2\%$) \\
      250 & $1.58$ ($-18.0\%$) & $1.58$ ($-17.7\%$) & $1.60$ ($-18.7\%$) & $4.32$ ($19.8\%$) & $4.33$ ($20.0\%$) & $4.48$ ($16.8\%$) \\
      500 & $1.89$ ($-24.2\%$) & $1.89$ ($-24.1\%$) & $1.94$ ($-24.7\%$) & $5.05$ ($14.4\%$) & $5.05$ ($14.0\%$) & $5.22$ ($11.8\%$) \\
      1000 & $2.28$ ($-30.5\%$) & $2.28$ ($-30.6\%$) & $2.34$ ($-31.0\%$) & $5.77$ ($9.8\%$) & $5.77$ ($9.3\%$) & $5.95$ ($8.1\%$) \\
      \bottomrule
    \end{tabular}
  }
\end{table}

\paragraph{Frozen vs Free Drafters.}
A comparison between the Frozen and Free Drafter is given in \cref{tab:frozen-compare}. Consistently, Frozen Drafter has superior block efficiency. Despite this, it is always worse in terms of its wall-clock speedups than its Free Drafter counterpart: the increase in block efficiency does not overcome the wall-clock cost of a drafter whose cost scales with the primary model. We note that the speedup from using the Free Drafter over Frozen Drafter is significant even for smaller number of denoising steps.

\section{Discussion}
We propose an algorithm which allows us to accelerate speculative diffusions via block verification. To do so, we present a novel method for sampling from the residual distribution which is commonly utilized for LLM speculative decoding.
Furthermore, we analyze Free Drafter---a heuristic approximation of a Frozen Drafter derived from extrapolating a single call of the primary model.
Practically, this provides a self-speculative sampling approach for speculative diffusions that improves upon the previously explored reflection coupling approach by up to $6.3\%$ in the measured latency speedup, with smaller gains when verification overhead dominates. From our experiments, we find that speculative diffusions provide better speedups as the number of denoising steps increases. Thus, if one utilized speculative diffusions alongside other acceleration methods that decrease the number of denoising steps, the overall speedup that speculative diffusions provide will be comparatively smaller.
As we are working within the speculative diffusions framework, we note that our acceleration approach is not applicable to deterministic samplers and that the draft model must have the same denoising schedule as the primary model~\citep{bortoli2025accelerated}. Another limitation is that, similar to other parallelization methods~\citep{shih2023parallel,chen2024accelerating}, verification increases memory overhead due to the parallel call of the primary model.

Finally, as noted by \cite{bortoli2025accelerated}, speculative sampling can also be used to accelerate simple Langevin dynamics integrators, and the block verification extension proposed here can be readily applied to this setup. Recently, \cite{kosmala2026speculativesamplingfastermolecular} extended the algorithm in \citet{bortoli2025accelerated} to several sophisticated second-order Langevin integrators used in molecular dynamics \citep{leimkuhler2016molecular}. We conjecture that block verification can also be extended to these integrators.

\bibliography{main}
\bibliographystyle{plainnat}

\newpage

\appendix

\startcontents[appendix]
\begingroup
\setcounter{tocdepth}{1}
\printcontents[appendix]{}{1}{
\section*{Appendix Contents}}
\endgroup
\clearpage

\section{Practical Considerations}%

\subsection{Sampling the Block Verification Residual}%
\label{app:u_sampling}

A key part of our block verification algorithm (\cref{algo:block-verification}) is sampling $\Psi_{\alpha}$ (\cref{eq:u-cdf}) to sample a correct residual function. Practically, we leverage a bisection search. In addition, we utilize a doubling procedure to find appropriate bounds for our search space.
Notice that the condition of $\Psi_{\alpha}(u) \neq 1$ already gives our bisection search an upper bound. So we only need to heuristically find an appropriate lower bound.
This is summarized in \cref{algo:u-sample}.

\begin{algorithm}
  \caption{$\texttt{USampleViaBisection}(\Vert \Delta \Vert, \alpha)$
  }\label{algo:u-sample}
  \begin{algorithmic}[1]
    \Require $\Vert \Delta \Vert > 0$, $\alpha \in (0, 1]$, error tolerance $\texttt{tol} > 0$

    \State Sample $u \sim \Unif[0, 1]$.
    \LineComment{Bounds.}
    \State Initialize bounds $(u_\textrm{lower}, u_\textrm{upper}) = (\log(\alpha)/\|\Delta\|+\|\Delta\|/2-1, \log(\alpha)/\|\Delta\|+\|\Delta\|/2)$.
    \While{$\Psi_{\alpha}(u_\textrm{lower}) > u$}
    \State Set $u_{\mathrm{lower}}=u_{\mathrm{upper}}-2(u_{\mathrm{upper}}-u_{\mathrm{lower}})$.
    \EndWhile
    \LineComment{Bisection search.}
    \State Initialize $u_\textrm{mid} = (u_\textrm{lower} + u_\textrm{upper}) / 2$.
    \While{$\vert \Psi_{\alpha}(u_\textrm{mid}) - u  \vert > \texttt{tol}$}
    \If{$ \Psi_{\alpha}(u_{\textrm{mid}}) < u$}
    \State Set $u_\textrm{lower} = u_\textrm{mid}$.
    \Else
    \State Set $u_\textrm{upper} = u_\textrm{mid}$.
    \EndIf
    \State Re-calculate $u_\textrm{mid} = (u_\textrm{lower} + u_\textrm{upper}) / 2$.
    \EndWhile
    \State \Return $u_\textrm{mid}$.
  \end{algorithmic}
\end{algorithm}

\subsection{Density ratio calculations}

In the verification algorithms presented (\cref{algo:token-verification,algo:block-verification}), we calculate density ratios of distributions (Gaussians). 
The algorithms are written in probability space to match the standard speculative decoding presentation. However, in practice, 
for additional numerical stability one can calculate these ratios in log space. 
Notice that
\begin{equation}
\log \frac{\mathcal{N}(\Z + \Delta; 0, \Id)}{\mathcal{N}(\Z; 0, \Id)} = - \Z^\top \Delta - \frac{1}{2} \Vert \Delta \Vert^2.
\end{equation}
Thus, for example, the block recursion in \cref{algo:block-verification} can be simplified to
\begin{equation}
    \log \alpha_j 
    = \min \left\{ 0, \log \alpha_{j-1} - \Z_j^\top \Delta_j - \frac{1}{2} \Vert \Delta_j \Vert^2 \right\}.
\end{equation}
Likewise, when computing the differences of scalar {\cdf}s $\alpha \Phi(a) - \Phi(b)$, one should utilize log-\cdf primitives (\eg, \verb+scipy.stats.norm.logcdf+) alongside other stable operators, \eg, using \verb+numpy.log1p+ to define a stable log-sub-exp operation, for numerical stability.

\section{Full Block Verification Algorithm}

We present the full block verification speculative diffusion pseudo-code in \cref{algo:full-block-verification}.

\begin{algorithm}[H]
  \caption{Speculative Diffusions with Block Verification}\label{algo:full-block-verification}
  \begin{algorithmic}[1]
    \Require Target model $q$, drafter model $p$, draft size $\gamma$, number of sampling steps $K$.
    \State Sample $\Y_0 \sim \mathcal{N}(0, \Id)$ and set $k = 0$.
    \While{$k < K$}
    \State Set $\gamma_k = \min\{ \gamma, K - k\}$.
    \LineComment{Construct Draft.}
    \State Initialize draft $\hat{\Y}_{k} = \Y_k$.
    \For{$j \in \{1, \ldots, {\gamma_k}\}$}
    \State Sample $\hat{\Y}_{k+j} \sim p(\cdot \mid \hat{\Y}_{k:k+j-1})$.
    \Comment{Cache draft mean $m_{k+j-1}^p$.}
    \EndFor
  \item[]
    \LineComment{Verify Draft.}
    \State Compute $m^q_{k+j-1} = m_{k+j-1}^q(\hat{\Y}_{k+j-1})$ for all $ j \in \{ 1, \ldots, \min\{ \gamma_k + 1, K - k \} \}$ in parallel.
    \State Initialize $\bm{c}_0 = 1$ and $\alpha_0 = 1$.
    \For{$j \in \{ 1, \ldots, {\gamma_k} \}$}
    \State Set $\Delta_{j} = (m_{k+j-1}^p - m_{k+j-1}^q) / \sigma_{k+j-1}$ and $\Z_j = (\hat{\Y}_{k+j} - m_{k+j-1}^p) / \sigma_{k+j-1}$.
    \State Calculate $\alpha_j = \min\left\{1, \alpha_{j-1} \frac{\mathcal{N}(\Z_{j} + \Delta_{j}; 0, \Id)}{\mathcal{N}(\Z_j; 0, \Id)} \right\}$.
    \EndFor
    \For{$j \in \{ 1, \ldots, {\gamma_k} \}$}
    \If{$j \neq {\gamma_k}$}
    \State Calculate $h_j = {v_j}/({v_j + 1 - \alpha_j})$, where
    \Statex \qquad \qquad \qquad $
    v_j = \alpha_j \Phi\left( \frac{\log \alpha_j}{\Vert \Delta_{j+1} \Vert} + \frac{\Vert \Delta_{j+1} \Vert}{2}\right) - \Phi\left( \frac{\log \alpha_j}{\Vert \Delta_{j+1} \Vert} - \frac{\Vert \Delta_{j+1} \Vert}{2}\right).
    $
    \Else
    \State Set $h_{\gamma_k} = \alpha_{\gamma_k}$.
    \EndIf
    \State Flip coin $\bm{c}_j = \llbracket \eta < h_j \rrbracket$, where $\eta \sim \Unif[0, 1]$.
    \EndFor
    \State Set $\tau = 1 + \max\left\{ j \in \{ 0, 1, \ldots, {\gamma_k}\} \mid \bm{c}_j = 1 \right\}$.
  \item[]
    \LineComment{Residual Calculation.}
    \If {$k+\tau \leq K$}
    \If{$\tau = {\gamma_k} + 1$}
    \State Sample $\Y_{k + {\gamma_k} + 1} \sim q(\cdot \mid \hat{\Y}_{k+{\gamma_k}})$.
    \Else
    \State Set $\Delta = \Delta_{\tau}$, $\Z = \Z_{\tau}$, $m^q = m^q_{k+\tau-1}$, $\sigma = \sigma_{k+\tau-1}, \alpha_{\textrm{res}} = \alpha_{\tau - 1}$.
    \State Sample $\U \sim \psi_{\alpha_{\tau-1}}$ via $\texttt{USampleViaBisection}(\Vert \Delta \Vert, \alpha_{\textrm{res}})$ (\cref{algo:u-sample}).
    \State Sample $\G \sim \mathcal{N}(0,\Id)$.
    \State Set $\Z_\perp = \G - e(e^\top \G)$, where $e = \Delta / \Vert \Delta \Vert $ so that $\Z_\perp \sim \mathcal{N}(0, \Id - ee^\top)$.
    \State Set $\Y_{k+\tau} = m^q + \sigma (\U e + \Z_\perp)$.
    \EndIf
    \EndIf
    \State Set $\Y_{k+j} = \hat{\Y}_{k+j}$ for $j \in \{1, \ldots, \tau - 1\}$.
    \State Set $k = \min\{ k + \tau, K \}$.
    \EndWhile
    \State \Return $\Y_{0:K}$.
  \end{algorithmic}
\end{algorithm}

\section{General Results}

In the following, we consider a variety of general results. We prove these in the most general way possible, \ie, with temperature and with ``prior'' probability $\alpha$.

\paragraph{Accept Ratio}
The acceptance probability of speculative diffusion is calculated via clamping a density ratio.
It is often useful to consider under what domains we should clamp the ratio.

\begin{lemma} \label{lem:no-clamp-coin}
  Let $p^\omega(y) = \mathcal{N}(m^p, \omega^2 \sigma^2 \Id)$ and $q^\omega(y) = \mathcal{N}(m^q, \omega^2 \sigma^2 \Id)$.
  Denote $\Delta = (m^p - m^q) / \sigma$ and $e = \Delta / \Vert \Delta \Vert$.
  Furthermore, let $\alpha \in (0, 1]$ and $\Vert \Delta \Vert > 0$.

  Then for $y = m^p + \sigma z^p$, we have
  \begin{equation}
    \alpha \frac{q^\omega(y)}{p^\omega(y)} < 1
    \quad \iff \quad
    (z^p)^\top e > \frac{\omega^2 \log \alpha}{\Vert \Delta \Vert} - \frac{\Vert \Delta \Vert}{2};
  \end{equation}
  and for $y = m^q + \sigma z^q$, we have
  \begin{equation}
    \alpha \frac{q^\omega(y)}{p^\omega(y)} < 1
    \quad \iff \quad
    (z^q)^\top e > \frac{\omega^2 \log \alpha}{\Vert \Delta \Vert} + \frac{\Vert \Delta \Vert}{2}.
  \end{equation}
\end{lemma}
\begin{proof}
  After considering the change of variables $y = m^p + \sigma z$, we have:
  \begin{align*}
    \alpha \frac{\mathcal{N}(z + \Delta; 0, \omega^2 \Id)}{\mathcal{N}(z; 0, \omega^2 \Id)} < 1
    \quad
    \iff & \quad
    \alpha \exp\left( -\frac{1}{2\omega^2} \left( \Vert \Delta \Vert^2 + 2 z^\top \Delta \right) \right) < 1 \\
    \iff & \quad
    \log \alpha -\frac{1}{2\omega^2} \left( \Vert \Delta \Vert^2 + 2 z^\top \Delta \right) < 0 \\
    \iff & \quad
    2\omega^2 \log \alpha -  \Vert \Delta \Vert^2 - 2 z^\top \Delta < 0.
  \end{align*}
  This yields the stated constraint. The other case follows identically with a modified sign.
\end{proof}

Typically, we will leave the superscript on $z$ implicit (where the most common choice is $z = z^p$ as this corresponds to the noise of the draft model in the speculative diffusions algorithm).

\paragraph{Non-zero Residual}
Additionally, the residual term depends on clamping the difference of distributions. It is also useful to know when this difference is non-zero.

\begin{corollary} \label{cor:non-zero-res}
  Let $p(y) = \mathcal{N}(m^p, \sigma^2 \Id)$ and $q(y) = \mathcal{N}(m^q, \sigma^2 \Id)$.
  Denote $\Delta = (m^p - m^q) / \sigma$ and $e = \Delta / \Vert \Delta \Vert$.
  Furthermore, let $\alpha \in (0, 1]$ and $\Vert \Delta \Vert > 0$.

  Then for $y = m^p + \sigma z^p$, we have
  \begin{equation}
    \alpha q(y) - p(y) > 0
    \quad \iff \quad
    (z^p)^\top e < \frac{\log \alpha}{\Vert \Delta \Vert} - \frac{\Vert \Delta \Vert}{2};
  \end{equation}
  and for $y = m^q + \sigma z^q$, we have
  \begin{equation}
    \label{eq:non-zero-res-mq}
    \alpha q(y) - p(y) > 0
    \quad \iff \quad
    (z^q)^\top e < \frac{\log \alpha}{\Vert \Delta \Vert} + \frac{\Vert \Delta \Vert}{2}.
  \end{equation}
\end{corollary}

\paragraph{Closed-form Acceptance Probability}
The sample acceptance probability can be calculated in closed form. This involves integrating the acceptance probability with the draft measure.

\begin{lemma}\label{lem:accept-prob}
  Let $p^\omega(y) = \mathcal{N}(m^p, \omega^2 \sigma^2 \Id)$ and $q^\omega(y) = \mathcal{N}(m^q, \omega^2 \sigma^2 \Id)$ with $p = p^{\omega = 1}$. Furthermore, let $\alpha \in (0, 1]$ and $\Vert \Delta \Vert > 0$.
  Then
  \begin{align}
    P_\textup{accept} &= \int p(y) \min \left\{ 1, \alpha \frac{q^\omega(y)}{p^\omega(y)} \right\} \, \dd y \nonumber \\
    &= \Phi\left( \frac{\omega^2 \log \alpha}{\Vert \Delta \Vert} - \frac{\Vert \Delta \Vert}{2} \right)
    + \alpha \exp\left( \frac{\Vert \Delta \Vert^2}{2 \omega^2} \left( \frac{1}{\omega^2} - 1 \right) \right) \left( 1 - \Phi\left( \frac{\omega^2 \log \alpha}{\Vert \Delta \Vert} - \frac{\Vert \Delta \Vert}{2} + \frac{\Vert \Delta \Vert}{\omega^2} \right) \right),
  \end{align}
  where $\Phi$ is the standard normal \cdf.
\end{lemma}
\begin{proof}
  We consider a change of variables of $y = m^p + \sigma z$. Then we have
  \begin{align*}
    P_\textup{accept} = \int \mathcal{N}(z, 0, \Id) \min \left\{ 1, \alpha \frac{\mathcal{N}(z + \Delta; 0, \omega^2 \Id)}{\mathcal{N}(z; 0, \omega^2 \Id)} \right\} \, \dd z.
  \end{align*}
  We will evaluate this in cases depending on which value the minimum is evaluated as.
  From \cref{lem:no-clamp-coin}, the second argument of the min is evaluated when
  \begin{equation*}
    z^\top e > \frac{\omega^2 \log \alpha}{\Vert \Delta \Vert} - \frac{\Vert \Delta \Vert}{2} = A.
  \end{equation*}
  We define the region $S = \{ z : z^\top e > A \}$.

  We now have,
  \begin{align*}
    P_\textup{accept}
    &=
    \int_{S^{c}} \mathcal{N}(z; 0, \Id) \, \dd z
    +
    \alpha \int_{S} \mathcal{N}(z; 0, \Id)
    \frac{\mathcal{N}(z + \Delta; 0, \omega^2 \Id)}{\mathcal{N}(z; 0, \omega^2 \Id)} \, \dd z \\
    &=
    \int_{S^{c}} \mathcal{N}(z; 0, \Id) \, \dd z
    +
    \alpha \int_{S} \mathcal{N}(z; 0, \Id) \exp\left( -\frac{\Vert \Delta \Vert}{2\omega^2} \left( \Vert \Delta \Vert + 2 z^\top e \right) \right)
    \, \dd z .
  \end{align*}

  Now we use a change of coordinates $z = u e + z_{\perp}$, where $u = e^\top z$. It can be shown that the resulting Jacobian is equal to $1$. Hence,
  \begin{align*}
    P_\textup{accept}
    &=
    \int \int^A_{-\infty} \mathcal{N}(u e + z_\perp; 0, \Id) \,  \dd u \, \dd z_\perp \\
    &\quad
    +
    \alpha \int \int_A^{\infty} \mathcal{N}(u e + z_\perp; 0, \Id) \exp\left( -\frac{\Vert \Delta \Vert}{2\omega^2} \left( \Vert \Delta \Vert + 2 u \right) \right) \, \dd u \, \dd z_\perp \\
    &=
    \int \int^A_{-\infty} \mathcal{N}(u; 0, 1)\mathcal{N}(z_\perp; 0, \Id) \,  \dd u \, \dd z_\perp \\
    &\quad
    +
    \alpha \int \int_A^{\infty} \mathcal{N}(u; 0, 1) \mathcal{N}(z_\perp; 0, \Id) \exp\left( -\frac{\Vert \Delta \Vert}{2\omega^2} \left( \Vert \Delta \Vert + 2 u \right) \right) \, \dd u \, \dd z_\perp  \\
    &=
    \int^A_{-\infty} \mathcal{N}(u; 0, 1) \,  \dd u
    +
    \alpha \int_A^{\infty} \mathcal{N}(u; 0, 1) \exp\left( -\frac{\Vert \Delta \Vert}{2\omega^2} \left( \Vert \Delta \Vert + 2 u \right) \right) \, \dd u \\
    &=
    \int^A_{-\infty} \mathcal{N}(u; 0, 1) \,  \dd u
    +
    \frac{\alpha}{\sqrt{2\pi}} \exp\left( - \frac{\Vert \Delta \Vert^2}{2 \omega^2} \right) \int_A^{\infty} \exp\left( -\frac{1}{2} \left( u^2 +\frac{2\Vert \Delta \Vert u}{\omega^2} \right) \right) \, \dd u \\
    &=
    \int^A_{-\infty} \mathcal{N}(u; 0, 1) \,  \dd u
    +
    \frac{\alpha}{\sqrt{2\pi}} \exp\left( - \frac{\Vert \Delta \Vert^2}{2 \omega^2} \right) \int_A^{\infty} \exp\left( -\frac{1}{2} \left( u + \frac{\Vert \Delta \Vert}{\omega^2} \right)^2 + \frac{\Vert \Delta \Vert^2}{2\omega^4} \right) \, \dd u \\
    &=
    \int^A_{-\infty} \mathcal{N}(u; 0, 1) \,  \dd u
    +
    \frac{\alpha}{\sqrt{2\pi}} \exp\left( - \frac{\Vert \Delta \Vert^2}{2 \omega^2} + \frac{\Vert \Delta \Vert^2}{2\omega^4} \right) \int_A^{\infty} \exp\left( -\frac{1}{2} \left( u + \frac{\Vert \Delta \Vert}{\omega^2} \right)^2 \right) \, \dd u \\
    &=
    \int^A_{-\infty} \mathcal{N}(u; 0, 1) \,  \dd u
    +
    {\alpha}\exp\left( \frac{\Vert \Delta \Vert^2}{2 \omega^2} \left( \frac{1}{\omega^2} - 1 \right) \right) \int_A^{\infty} \mathcal{N}\left(u + \frac{\Vert \Delta \Vert}{\omega^2}; 0, 1\right) \, \dd u \\
    &=
    \int^A_{-\infty} \mathcal{N}(u; 0, 1) \,  \dd u
    +
    {\alpha}\exp\left( \frac{\Vert \Delta \Vert^2}{2 \omega^2} \left( \frac{1}{\omega^2} - 1 \right) \right) \int_{A + \frac{\Vert \Delta \Vert}{\omega^2}}^{\infty} \mathcal{N}\left(u; 0, 1\right) \, \dd u.
  \end{align*}
  Finally, writing the solution in terms of the standard normal \cdf finishes the proof.
\end{proof}

\paragraph{Orthogonal Decomposition of Noise}
A key technique we use in our proofs is the decomposition of a random variable $\Z \sim \mathcal{N}(0; \Id)$ (in $\mathbb{R}^d$).

\begin{lemma}
  \label{lem:orthogonal-breakdown}
  Let $e$ be a unit vector and $\Z \sim \mathcal{N}(0; \Id)$. Define the orthogonal decomposition $\Z = \U e + \Z_\perp$, where we have $\U = e^\top \Z$ and $\Z_\perp = (\Id - e e^\top) \Z$.
  We have that
  \begin{equation}
    \U \sim \mathcal{N}(0, 1) \qquad \textup{and} \qquad \Z_\perp \sim \mathcal{N}(0, \Id - e e^\top).
  \end{equation}
  Moreover, $\U$ and $\Z_\perp$ are independent.
\end{lemma}
\begin{proof}
  Given the orthogonal decomposition $\Z = \U e + \Z_\perp$, where we have $\U = e^\top \Z$ and $\Z_\perp = (\Id - e e^\top) \Z$. We shorthand $P_\perp = \Id - e e^\top$.

  We show that the random variable $\U$ has the following properties
  \begin{align*}
    \expect[\U] &= \expect[e^\top \Z] = e^\top \expect[\Z] = 0 \\
    \cov[\U] &= \expect[\U^2] - \expect[\U]^2 = \expect[(e^\top \Z)^2] = e^\top \expect[\Z \Z^\top] e = e^\top (\Id) e = 1.
  \end{align*}
  Furthermore, as $\U$ is a linear transform of a Gaussian random variable, we have that $\U \sim \mathcal{N}(0, 1)$.

  Similarly, we consider the properties of $\Z_\perp$
  \begin{align*}
    \expect[\Z_\perp] &= \expect[P_\perp \Z] = P_\perp \expect[\Z] = 0 \\
    \cov[\Z_\perp] &= \expect[\Z_\perp \Z_\perp^\top] - \expect[\Z_\perp] \expect[\Z_\perp]^\top = \expect[\Z_\perp \Z_\perp^\top] = P_\perp \expect[\Z \Z^\top ]P_\perp^\top = P_\perp P_\perp^\top.
  \end{align*}
  We further note that $P_\perp$ is a symmetric and idempotent matrix. That is
  \begin{align*}
    P_\perp P_\perp^\top = P_\perp P_\perp = (\Id - e e^\top) (\Id - e e^\top) = \Id - e e^\top - e e^\top + e e^\top e e^\top = \Id - e e^\top = P_\perp.
  \end{align*}
  Thus $\cov[\Z_\perp] = P_\perp$. Similar to $\U$, $\Z_\perp$ is a linear transform of $\Z$, thus $\Z_\perp \sim \mathcal{N}(0, P_\perp)$.

  Finally, we compute the cross-covariance:
  \begin{align*}
    \cov(\U,\Z_\perp)
    &= \expect[\U \Z_\perp^\top] - \expect[\U]\expect[\Z_\perp]^\top \\
    &= \expect[(e^\top \Z)(P_\perp \Z)^\top] \\
    &= e^\top \expect[\Z \Z^\top] P_\perp^\top \\
    &= e^\top P_\perp \\
    &= e^\top(\Id - e e^\top) \\
    &= 0.
  \end{align*}
  Therefore $(\U, \Z_\perp)$ is jointly Gaussian with zero cross-covariance, and hence $\U$ and $\Z_\perp$ are independent.
\end{proof}

\paragraph{Orthogonal Change of Variable}
A useful property of the orthogonal decomposition is that the corresponding change of coordinates has the following property.

\begin{lemma}
  \label{lem:orthogonal-integral-change}
  Let $e \in \mathbb{R}^d$ be a unit vector. Choose a matrix $E_\perp \in \mathbb{R}^{d \times (d-1)}$ whose columns form an orthonormal basis of $e^\perp = \{ x \in \mathbb{R}^d \mid e^\top x = 0 \}$. Then the map $T \colon \mathbb{R} \times \mathbb{R}^{d-1} \to \mathbb{R}^d$ given by $(u, v) \mapsto ue + E_\perp v$ is an orthogonal linear isomorphism.

  Hence for every integrable  $f \colon \mathbb{R}^d \to \mathbb{R}$,
  \begin{equation}
    \int_{\mathbb{R}^d} f(z) \, \dd z
    = \int_{\mathbb{R}} \int_{\mathbb{R}^{d-1}} f(ue + E_\perp v) \, \dd v \, \dd u.
  \end{equation}
\end{lemma}
\begin{proof}
  Let $Q = [e, E_\perp] \in \mathbb{R}^{d\times d}$. Since $e$ is a unit vector and $E_\perp$ is an orthonormal basis of $e^{\perp}$, $Q$ is orthogonal, \ie, $Q^\top Q = \Id$. Thus $\vert \det Q \vert = 1$.

  Thus, every $z \in \mathbb{R}^d$ can be (uniquely) written as $z = ue + E_\perp v$, where $u \in \mathbb{R}$ and $v \in \mathbb{R}^{d-1}$.
  The result then follows from the usual change of variables formula:

  \begin{equation*}
    \int_{\mathbb{R}^d} f(z) \, \dd z
    =
    \int_{\mathbb{R}}\int_{\mathbb{R}^{d-1}}
    f(ue+E_\perp v) \, |\det Q| \, \dd v \, \dd u
    =
    \int_{\mathbb{R}}\int_{\mathbb{R}^{d-1}}
    f(ue+E_\perp v) \, \dd v \, \dd u.
  \end{equation*}
  As required.
\end{proof}

Combining \cref{lem:orthogonal-breakdown,lem:orthogonal-integral-change}, we get the following property.

\begin{corollary}%
  \label{lem:orthogonal-expect-change}
  Let $f \colon \mathbb{R}^d \to \mathbb{R}$ and let $e$ be a normalized vector. We consider a change of variables $z \mapsto u e + z_\perp$ where $u = e^\top z$ and $z_\perp = (\Id - e e^\top) z$. We have that
  \begin{equation}
    \int f(z) \mathcal{N}(z; 0, \Id) \, \dd z = \iint f(ue + z_\perp) \mathcal{N}(u; 0, 1) \mathcal{N}(z_\perp; 0, \Id - e e^\top) \, \dd u \, \dd z_\perp.
  \end{equation}
\end{corollary}

\begin{remark}
  In \cref{lem:orthogonal-expect-change}, we abuse notation in the integration of $z_\perp$. More precisely, the integration should be \wrt the $(d-1)$-dimensional subspace $e^\perp = \{ x \in \mathbb{R}^d \mid e^\top x = 0 \}$.
  The integration $\dd z_\perp$ should be understood as the $(d-1)$ dimensional Lebesgue measure on $e^\perp$, and $\mathcal{N}(z_\perp; 0, \Id - e e^\top)$ denote the Gaussian measure supported on $e^\perp$.

  More precisely, if $E_\perp \in \mathbb{R}^{d \times (d-1)}$ has orthonormal columns spanning $e^\perp$, then writing $z_\perp = E_\perp v$ with $v \in \mathbb{R}^{d-1}$, the right-hand side of \cref{lem:orthogonal-expect-change} is rigorously defined as
  \begin{equation*}
    \int_{\mathbb{R}} \int_{\mathbb{R}^{d-1}}
    f(ue + E_\perp v) \mathcal{N}(u;0,1) \mathcal{N}(v;0,\Id_{d-1})\, \dd v \, \dd u.
  \end{equation*}
  For notational simplicity, in the sequel, we will utilize the simpler notation of \cref{lem:orthogonal-expect-change} and continue to write the integral in terms of $z_\perp$.
\end{remark}

\section{Decomposition of LLM Residual}

\begin{proposition}
  Let $p = \mathcal{N}(m^p, \sigma^2 \Id)$, $q = \mathcal{N}(m^q, \sigma^2 \Id)$, $\alpha \in (0, 1]$, and $\Vert \Delta \Vert > 0$. The distribution $r(y; \alpha) \propto \max\{ 0, \alpha q(y) - p(y) \}$ can be sampled by the following:
  \begin{align}
    \U &\sim \psi_{\alpha} \\
    \Z_\perp &\sim \mathcal{N}(0, \Id - ee^\top) \\
    \Y &= m^q + \sigma (\U e + \Z_\perp)
  \end{align}
  where
  \begin{equation}
    \psi_{\alpha}(u) \propto \max\{ 0, \alpha \mathcal{N}(u; 0, 1) - \mathcal{N}(u - \Vert \Delta \Vert; 0, 1) \}.
  \end{equation}
\end{proposition}
\begin{proof}
  Given an input $y$, we consider a reparametrization $y = m^q + \sigma z$. Thus for a random variable $\Y \sim r(\cdot; \alpha)$, we have that
  \begin{equation}
    \Y = m^q + \sigma \Z.
  \end{equation}
  We will examine the distribution of $\Z$. We denote the normalization term of $r(\cdot; \alpha)$ as  $\mathcal{Z}_\alpha$. As all but $\Z$ are constants here, we have
  \begin{align*}
    r(y; \alpha)
    &= \frac{1}{\mathcal{Z}_\alpha} \max\{ 0, \alpha q(y) - p(y) \} \\
    &= \frac{1}{\mathcal{Z}_\alpha} \max\{ 0, \alpha \mathcal{N}(y; m^q, \sigma^2 \Id) - \mathcal{N}(y; m^p, \sigma^2 \Id) \} \\
    &= \frac{1}{\mathcal{Z}_\alpha} \max\{ 0, \alpha \mathcal{N}(m^q + \sigma z; m^q, \sigma^2 \Id) - \mathcal{N}(m^q + \sigma z; m^p, \sigma^2 \Id) \} \\
    &= \frac{1}{\tilde{\mathcal{Z}}_\alpha} \max\{ 0, \alpha \mathcal{N}(z; 0, \Id) - \mathcal{N}(z - \Delta; 0, \Id) \},
  \end{align*}
  where $\tilde{\mathcal{Z}}_\alpha$ absorbs the ${1}/{\sigma}$ change-of-variables constant.

  With slight abuse of notation, we can now write $r(\cdot; \alpha)$ in terms of $z$ as
  \begin{equation}
    r(z; \alpha)
    \propto
    \mathcal{N}(z; 0, \Id)\max\left\{ 0, \alpha - \frac{\mathcal{N}(z - \Delta; 0, \Id)}{\mathcal{N}(z; 0, \Id)} \right\}.
  \end{equation}
  Notice that the ratio simplifies as
  \begin{align*}
    \frac{\mathcal{N}(z - \Delta; 0, \Id)}{\mathcal{N}(z; 0, \Id)}
    &= \exp\left( - \frac{1}{2} \left( \Vert z - \Delta \Vert^2 - \Vert z \Vert^2 \right) \right) \\
    &= \exp\left( \Delta ^\top z - \frac{\Vert \Delta \Vert^2}{2} \right) \\
    &= \exp\left( (e^\top z) \Vert \Delta \Vert - \frac{\Vert \Delta \Vert^2}{2} \right).
  \end{align*}
  We will now consider an orthogonal decomposition with $e = \Delta / \Vert \Delta \Vert$, using \cref{lem:orthogonal-breakdown} with $\Z = \U e + \Z_\perp$, for $\U = e^\top \Z$ and $\Z_\perp = P_\perp \Z$. We shorthand $P_\perp = \Id - e e^\top$.
  Notice that the ratio can be solely written in terms of $\U$. In particular, we get
  \begin{align*}
    r(z; \alpha)
    =
    r(u; \alpha)
    r(z_\perp; \alpha)
    &\propto
    \mathcal{N}(z; 0, \Id)\max\left\{ 0, \alpha - \frac{\mathcal{N}(z - \Delta; 0, \Id)}{\mathcal{N}(z; 0, \Id)} \right\} \\
    &=
    \mathcal{N}(u; 0, 1) \mathcal{N}(z_\perp; 0, P_\perp) \max\left\{ 0, \alpha - \exp\left( u \Vert \Delta \Vert - \frac{\Vert \Delta \Vert^2}{2} \right) \right\}.
  \end{align*}
  Notice that the sampling of $\Z_\perp$ is independent of $\U$. Hence we can sample $\Z_\perp$ separately.

  Now examining $\U$ and its \pdf $r(u; \alpha)$, we have
  \begin{align*}
    r(u; \alpha)
    &\propto
    \mathcal{N}(u; 0, 1) \max\left\{ 0, \alpha - \exp\left( u \Vert \Delta \Vert - \frac{\Vert \Delta \Vert^2}{2} \right) \right\} \\
    &= \max\left\{ 0, \alpha \mathcal{N}(u; 0, 1) - \mathcal{N}(u; 0, 1) \exp\left( u \Vert \Delta \Vert - \frac{\Vert \Delta \Vert^2}{2} \right) \right\}.
  \end{align*}
  Simplifying the last term, we get
  \begin{align*}
    \mathcal{N}(u; 0, 1) \exp\left( u \Vert \Delta \Vert - \frac{\Vert \Delta \Vert^2}{2} \right)
    &= \frac{1}{(2 \pi)^{1 / 2}} \exp\left( -\frac{u^2}{2} \right)
    \exp\left( u \Vert \Delta \Vert - \frac{\Vert \Delta \Vert^2}{2} \right) \\
    &= \frac{1}{(2 \pi)^{1 / 2}} \exp\left( -\frac{1}{2} \left( u^2 - 2u \Vert \Delta \Vert + \Vert \Delta \Vert^2 \right) \right) \\
    &= \frac{1}{(2 \pi)^{1 / 2}} \exp\left( -\frac{1}{2} \left( u - \Vert \Delta \Vert \right)^2 \right) \\
    &= \mathcal{N}(u - \Vert \Delta \Vert; 0, 1).
  \end{align*}

  Combining together, we have that $r(u; \alpha) = \psi_\alpha(u) \propto \max\{ 0, \alpha \mathcal{N}(u; 0, 1) - \mathcal{N}(u - \Vert \Delta \Vert; 0, 1) \}$.
  Further combining the reparameterization of $\Y$, independent sampling $\U$ and $\Z_\perp$, yields the desired result.
\end{proof}

\begin{lemma}
  Let $\psi_{\alpha}(u) \propto \max\{ 0, \alpha \mathcal{N}(u; 0, 1) - \mathcal{N}(u - v; 0, 1) \}$ for $\alpha \in (0, 1]$ and $ v > 0$.
  The \cdf of $\psi_{\alpha}$ is given by
  \begin{equation}
    \Psi_{\alpha}(u)
    =
    \begin{cases}
      \frac{\alpha \Phi(u) - \Phi(u - v)}{\alpha \Phi(\frac{\log \alpha}{v} + \frac{v}{2}) - \Phi(\frac{\log \alpha}{v} - \frac{v}{2})}
      & \textrm{if }  u < \frac{\log \alpha}{v} + \frac{v}{2} \\
      1 & \textrm{otherwise},
    \end{cases}
  \end{equation}
  where $\Phi$ corresponds to the \cdf of the standard normal distribution.
\end{lemma}
\begin{proof}
  The \pdf of $\psi_\alpha$ is given by
  \begin{equation}
    \psi_\alpha(u) = \frac{ \max \left\{ 0, \alpha \mathcal{N}(u; 0, 1) - \mathcal{N}(u - v; 0, 1) \right\}}{ \int_{-\infty}^\infty \max\left\{ 0, \alpha \mathcal{N}(u'; 0, 1) - \mathcal{N}(u' - v; 0, 1) \right\} \, \dd u' },
  \end{equation}
  where the denominator is strictly positive by definition.

  Note that from a modification of the proof of \cref{cor:non-zero-res}, $\psi_\alpha$ is non-zero iff
  \begin{align*}
    \alpha > \frac{\mathcal{N}(u - v; 0, 1)}{\mathcal{N}(u; 0, 1)}
    \iff  uv < \log \alpha + \frac{v^2}{2}
    \iff u < \frac{\log \alpha}{v} + \frac{v}{2}.
  \end{align*}
  Notably, the \cdf $\Psi_\alpha$ will be $1$ whenever $u$ is larger than this upper bound.
  Otherwise, we can consider the case where this bound holds.
  We have that the \cdf in this case will be
  \begin{align*}
    \Psi_\alpha(u)
    &= \frac{ \int_{-\infty}^u \max \left\{ 0, \alpha \mathcal{N}(u'; 0, 1) - \mathcal{N}(u' - v; 0, 1) \right\} \, \dd u' }{ \int_{-\infty}^\infty \max\left\{ 0, \alpha \mathcal{N}(u'; 0, 1) - \mathcal{N}(u' - v; 0, 1) \right\} \, \dd u' } \\
    &= \frac{ \int_{-\infty}^u \alpha \mathcal{N}(u'; 0, 1) - \mathcal{N}(u' - v; 0, 1) \, \dd u' }{ \int_{-\infty}^{\frac{\log \alpha}{v} + \frac{v}{2}} \alpha \mathcal{N}(u'; 0, 1) - \mathcal{N}(u' - v; 0, 1) \, \dd u' }.
  \end{align*}
  We calculate the numerator and denominator separately.

  For the numerator, we get
  \begin{align*}
    \int_{-\infty}^u \alpha \mathcal{N}(u'; 0, 1) - \mathcal{N}(u' - v; 0, 1) \, \dd u'
    &= \int_{-\infty}^u \alpha \mathcal{N}(u'; 0, 1) \, \dd u' - \int_{-\infty}^u \mathcal{N}(u' - v; 0, 1) \, \dd u' \\
    &= \alpha \Phi(u) - \Phi(u - v).
  \end{align*}

  Then similarly for the denominator,
  \begin{align*}
    &\int_{-\infty}^{\frac{\log \alpha}{v} + \frac{v}{2}} \alpha \mathcal{N}(u'; 0, 1) - \mathcal{N}(u' - v; 0, 1) \, \dd u' \\
    &=
    \int_{-\infty}^{\frac{\log \alpha}{v} + \frac{v}{2}} \alpha \mathcal{N}(u'; 0, 1) \, \dd u' - \int_{-\infty}^{\frac{\log \alpha}{v} + \frac{v}{2}} \mathcal{N}(u' - v; 0, 1) \, \dd u' \\
    &= \alpha \Phi \left( \frac{\log \alpha}{v} + \frac{v}{2} \right) - \Phi \left( \frac{\log \alpha}{v} - \frac{v}{2} \right).
  \end{align*}
  Together, alongside considering the constraint on $u$, yields the result.
\end{proof}

\subsection*{Examples of Residual Plot Comparison}

The residuals provided by \cref{algo:reflect-res,algo:decomposition-res} are fundamentally different.
Indeed, the reflection coupling gives a deterministic correction, while the decomposition corresponds to a distribution made from the difference of Gaussian distributions (normalized and non-zero). In \cref{fig:res-fixed-scale,fig:res-fixed-shift}, we show a 1D example of what the decomposition/LLM residual $r_\Gamma$ looks like for Gaussian distributions, \ie, the diffusion model case. \Cref{fig:res-fixed-scale} fixes the scale $c=1$ and varies the shift $v$, which corresponds to the mismatch between the drafter and primary models. \Cref{fig:res-fixed-shift} fixes $v=1$ and varies the scale $c$ (equivalently $\alpha$), which is relevant for block verification.

\begin{figure*}
  \centering
  \includegraphics[width=1.0\linewidth]{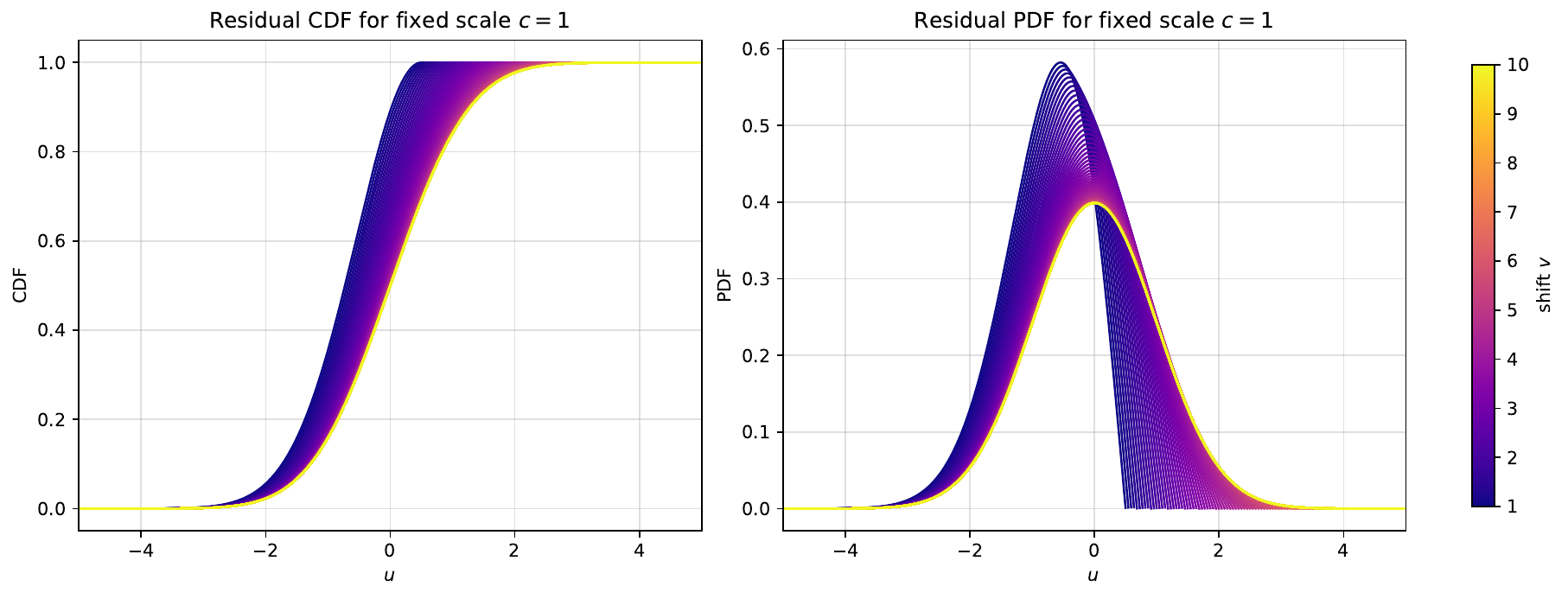}%
  \caption{Plot of the \pdf $f(u) \propto \max\{ 0, c \mathcal{N}(u; 0, 1) - \mathcal{N}(u - v; 0, 1) \}$ and its \cdf for fixed scale $c = 1$ over various shifts $v$. Normalization approximated via trapezoid integration.}%
  \label{fig:res-fixed-scale}
\end{figure*}

\begin{figure*}
  \centering
  \includegraphics[width=1.0\linewidth]{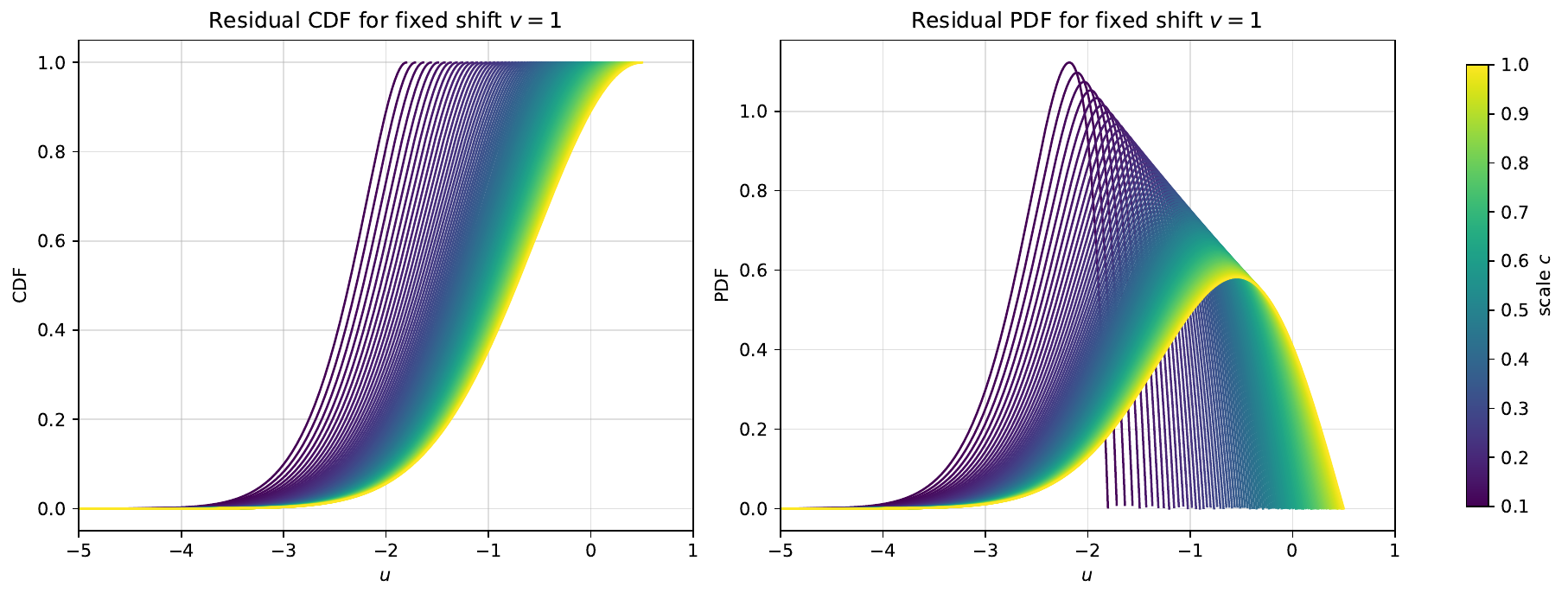}%
  \caption{Plot of the PDF $f(u) \propto \max\{ 0, c \mathcal{N}(u; 0, 1) - \mathcal{N}(u - v; 0, 1) \}$ and its \cdf for fixed shift $v = 1$ over various scales $c$. Normalization approximated via trapezoid integration.}%
  \label{fig:res-fixed-shift}
\end{figure*}

\section{Block Verification Proof}
\label{sec:block-verification-proof}

\begin{proposition}
  Given {the input of block verification \cref{algo:block-verification}}, the acceptance probability of block verification for $j\in\{1,\ldots,\gamma-1\}$ can be calculated as
  \begin{equation}
    h_j = \frac{\alpha_j \Phi\left( \frac{\log \alpha_j}{\Vert \Delta_{j+1} \Vert} + \frac{\Vert \Delta_{j+1} \Vert}{2} \right) - \Phi\left( \frac{\log \alpha_j}{\Vert \Delta_{j+1} \Vert} - \frac{\Vert \Delta_{j+1} \Vert}{2} \right)}{\alpha_j \Phi\left( \frac{\log \alpha_j}{\Vert \Delta_{j+1} \Vert} + \frac{\Vert \Delta_{j+1} \Vert}{2} \right) - \Phi\left( \frac{\log \alpha_j}{\Vert \Delta_{j+1} \Vert} - \frac{\Vert \Delta_{j+1} \Vert}{2} \right) + 1 - \alpha_j},
  \end{equation}
  where $\Delta_{j+1}$ and $\alpha_j$ are values calculated in a round of speculative diffusions, as per \cref{algo:block-verification}. For $j=\gamma$, the algorithm sets $h_\gamma=\alpha_\gamma$.
\end{proposition}
\begin{proof}
  Without loss of generality, we assume that we are examining step $j$ in a round starting at index $k$, \ie, we are considering the $k + j + 1$ sample. We denote the draft sequence by $\hat{y}_{k:k+j}$ up to this step.
  Through adapting \citet[Equation (5)]{sun2025block} to a continuous state space, we get
  \begin{equation}
    \label{eq:h_j_in_proof}
    h_j = \frac{\int_{\mathbb{R}^d} \max\{0, \alpha_j q(\hat{y} \mid \hat{y}_{k+j}) - p(\hat{y} \mid \hat{y}_{k:k+j})\} \, \dd \hat{y}}{\int_{\mathbb{R}^d} \max\{0, \alpha_j q(\hat{y} \mid \hat{y}_{k+j}) - p(\hat{y} \mid \hat{y}_{k:k+j})\} \, \dd \hat{y} + 1 - \alpha_j}
  \end{equation}
  where $d$ corresponds to the dimension of our model; we remind that the target model is Markov, but the draft is not necessarily Markov.

  We note that \cref{eq:h_j_in_proof} can be rewritten in terms of its denoising difference $\Delta_{j+1}$ (which depends on ``history'' $\hat{y}_{k:k+j}$), its variance $\sigma_{k+j}^2$, and the sampled noise of the drafter at that step $z_{k+j}$. For simplicity, we will drop all indices.

  Key to our derivation is the simplification of the numerator.
  To do so, we focus on the region in which the inner term of the integral is positive by utilizing \cref{cor:non-zero-res}. In particular, we utilize a change of coordinates via $\hat{y} = m^p + \sigma z$. We designate $S = \{ z : z^\top e < A \}$, where $A = \log \alpha / \Vert \Delta \Vert - \Vert \Delta \Vert / 2$. Thus we have,
  \begin{align*}
    \int_{\mathbb{R}^d} \max\{ 0, \alpha q(\hat{y} \mid \hat{y}_{k+j}) - p(\hat{y} \mid \hat{y}_{k:k+j}) \} \, \dd \hat{y}
    &= \int_{S} \alpha \mathcal{N}(\hat{y}; m^q, \sigma^2 \Id) - \mathcal{N}(\hat{y}; m^p, \sigma^2 \Id) \, \dd \hat{y} \\
    &= \int_{S} \alpha \mathcal{N}(z + \Delta; 0, \Id) - \mathcal{N}(z; 0, \Id) \, \dd z.
  \end{align*}
  Now we consider the orthogonal decomposition of $z = (e^\top z) e + z_\perp$, noting that $e^\top z$ appears in the condition of $S$. We denote $u = e^\top z$. Thus, we can simplify
  \begin{align*}
    & \int_{S} \alpha \mathcal{N}(z + \Delta; 0, \Id) - \mathcal{N}(z; 0, \Id) \, \dd z \\
    &= \int_{S} \alpha \mathcal{N}(u e + \Vert \Delta \Vert e + z_\perp; 0, \Id) - \mathcal{N}(u e + z_\perp; 0, \Id) \, \dd z \\
    &= \iint_{S} \alpha \mathcal{N}(u e + \Vert \Delta \Vert e + z_\perp; 0, \Id) - \mathcal{N}(u e + z_\perp; 0, \Id) \, \dd u \, \dd z_\perp \\
    &= \int_{-\infty}^{A} \int_{\mathbb{R}^d} \alpha \mathcal{N}((u + \Vert \Delta \Vert) e + z_\perp; 0, \Id) - \mathcal{N}(u e + z_\perp; 0, \Id) \, \dd z_\perp \, \dd u.
  \end{align*}
  We shorthand $P_\perp = \Id - e e^\top$. Now, breaking down each integral via \cref{lem:orthogonal-breakdown}, we have
  \begin{align*}
    &\int_{-\infty}^{A} \int_{\mathbb{R}^d} \alpha \mathcal{N}((u + \Vert \Delta \Vert) e + z_\perp; 0, \Id) - \mathcal{N}(u e + z_\perp; 0, \Id) \, \dd z_\perp \, \dd u \\
    &=
    \alpha
    \int_{-\infty}^{A} \int_{\mathbb{R}^d} \mathcal{N}((u + \Vert \Delta \Vert) e + z_\perp; 0, \Id) \, \dd z_\perp \, \dd u
    -
    \int_{-\infty}^{A} \int_{\mathbb{R}^d} \mathcal{N}(u e + z_\perp; 0, \Id) \, \dd z_\perp \, \dd u \\
    &=
    \alpha
    \int_{-\infty}^{A+\Vert \Delta \Vert} \int_{\mathbb{R}^d} \mathcal{N}( u e + z_\perp; 0, \Id) \, \dd z_\perp \, \dd u
    -
    \int_{-\infty}^{A} \int_{\mathbb{R}^d} \mathcal{N}(u e + z_\perp; 0, \Id) \, \dd z_\perp \, \dd u \\
    &=
    \alpha
    \int_{-\infty}^{A+\Vert \Delta \Vert} \int_{\mathbb{R}^d} \mathcal{N}( u; 0, 1)\mathcal{N}(z_\perp; 0, P_\perp) \, \dd z_\perp \, \dd u
    -
    \int_{-\infty}^{A} \int_{\mathbb{R}^d} \mathcal{N}( u; 0, 1)\mathcal{N}(z_\perp; 0, P_\perp) \, \dd z_\perp \, \dd u \\
    &=
    \alpha
    \int_{-\infty}^{A+\Vert \Delta \Vert} \mathcal{N}( u; 0, 1) \, \dd u
    -
    \int_{-\infty}^{A} \mathcal{N}( u; 0, 1) \, \dd u \\
    &= \alpha \Phi(A + \Vert \Delta \Vert) - \Phi(A).
  \end{align*}
  Substituting our definition of $A$, we get
  \begin{equation*}
    \int_{\mathbb{R}^d} \max\{0, \alpha q(\hat{y} \mid \hat{y}_{k+j}) - p(\hat{y} \mid \hat{y}_{k:k+j})\} \, \dd \hat{y}
    =
    \alpha \Phi\left( \frac{\log \alpha}{\Vert \Delta \Vert} + \frac{\Vert \Delta \Vert}{2} \right) - \Phi\left( \frac{\log \alpha}{\Vert \Delta \Vert} - \frac{\Vert \Delta \Vert}{2} \right).
  \end{equation*}
  Plugging this simplification into our equation of $h_j$ (with $\Delta = \Delta_{j+1}$ and $\alpha = \alpha_{j}$) proves the result.
\end{proof}

\section{Limits of Reflection-Style Deterministic Residuals for Block Verification}
\label{app:impossibility}

\begin{proposition}%
  \label{prop:impossibility_app}
  Fix $\gamma \geq 2$. There does not exist a valid deterministic residual correction of the form $t(\Z)$ for block verification that simultaneously
  \begin{enumerate}
    \item is invertible; and
    \item admits a constant Jacobian factor under change of variables.
  \end{enumerate}
\end{proposition}
\begin{proof}
  Suppose that $\gamma \geq 2$ and $c_{\gamma - 1} = 1$. To derive a contradiction, assume there exists a deterministic residual correction $t$ satisfying (i) and (ii).
  We note that for a residual to be valid, we additionally require $t$ to yield the primary model $q$ for every possible value of $\alpha_{\gamma-1} \in (0,1]$.

  From our assumptions, for the $\gamma$ index, we have that
  \begin{equation}
    \Y_{k+\gamma} = m^q_{k+\gamma-1} + \sigma_{k+\gamma-1}
    \begin{cases}
      \Z + \Delta_{\gamma} & \textrm{with probability } h_{\gamma}(\Z) \\
      t(\Z) & \textrm{with probability } 1 - h_{\gamma}(\Z)
    \end{cases},
  \end{equation}
  where with some abuse of notation we define $h_{\gamma}(z) = \min\left\{ 1, \alpha_{\gamma - 1} \mathcal{N}(z + \Delta_{\gamma}; 0 , \Id) / \mathcal{N}(z; 0 , \Id) \right\}$ (note that $h_{\gamma} = \alpha_{\gamma}$ in \cref{algo:block-verification}), which depends on the randomization of the last draft element.
  We consider where the draft and target diffusions are non-trivially different:  $\Vert \Delta_{\gamma} \Vert \neq 0$.
  To ensure that the output of block verification follows the distribution of $q$, we require
  \begin{equation}
    \label{eq:correction-law}
    \begin{cases}
      \Z + \Delta_{\gamma} & \textrm{with probability } h_{\gamma}(\Z) \\
      t(\Z) & \textrm{with probability } 1 - h_{\gamma}(\Z)
    \end{cases}
    \sim  \mathcal{N}(0, \Id).
  \end{equation}

  Let us evaluate the law of the left-hand side. By assumption (i), the inverse map $t^{-1}$ exists, and by assumption (ii) the change of variables through $t^{-1}$ contributes a constant Jacobian factor $C > 0$.
  \begin{align*}
    &
    \int \llbracket z = z' + \Delta_{\gamma} \rrbracket \mathcal{N}(z'; 0, \Id) h_{\gamma}(z') \dd z'
    +
    \int \llbracket z = t(z') \rrbracket \mathcal{N}(z'; 0, \Id) (1 - h_{\gamma}(z')) \dd z' \\
    &=
    \mathcal{N}(z - \Delta_{\gamma}; 0, \Id) h_{\gamma}(z- \Delta_{\gamma})
    +
    C \mathcal{N}(t^{-1}(z); 0, \Id) (1 - h_{\gamma}(t^{-1}(z))) \\
    &=
    \min\left\{ \mathcal{N}(z - \Delta_{\gamma}; 0, \Id), \alpha_{\gamma - 1}\mathcal{N}(z; 0, \Id) \right\} \\
    & \quad +
    C \max \left\{ 0, \mathcal{N}(t^{-1}(z); 0, \Id) - \alpha_{\gamma - 1} \mathcal{N}(t^{-1}(z) + \Delta_{\gamma}; 0, \Id) \right\},
  \end{align*}
  where $C > 0$ is the constant Jacobian factor.
  Now matching it with the law on the RHS, we have the equation
  \begin{align}
    &\max\left\{ \mathcal{N}(z; 0, \Id) - \mathcal{N}(z - \Delta_{\gamma}; 0, \Id), (1 - \alpha_{\gamma - 1})\mathcal{N}(z; 0, \Id) \right\} \nonumber \\
    &\quad =
    C \max \left\{ 0, \mathcal{N}(t^{-1}(z); 0, \Id) - \alpha_{\gamma - 1} \mathcal{N}(t^{-1}(z) + \Delta_{\gamma}; 0, \Id) \right\}.
  \end{align}
  For this to be a valid residual, we require this to be true for all $z$ and for every $\alpha_{\gamma - 1} \in (0,1)$.

  Suppose $\alpha_{\gamma - 1} < 1$. Then the LHS is strictly positive, \ie, $(1 - \alpha_{\gamma - 1})\mathcal{N}(z; 0, \Id) > 0$.
  Thus, we require for all $z$
  \begin{equation}
    \max \left\{ 0, \mathcal{N}(t^{-1}(z); 0, \Id) - \alpha_{\gamma - 1} \mathcal{N}(t^{-1}(z) + \Delta_{\gamma}; 0, \Id) \right\} > 0.
  \end{equation}
  Hence, we require $\mathcal{N}(t^{-1}(z); 0, \Id) > \alpha_{\gamma - 1} \mathcal{N}(t^{-1}(z) + \Delta_{\gamma}; 0, \Id)$ for all $z$. This can be shown to be equivalent to the condition
  \begin{equation}
    \label{eq:deterministic-cond}
    \exp\left( \frac{1}{2} (2 \Delta_{\gamma}^{\T}t^{-1}(z) + \Vert \Delta_{\gamma} \Vert^2) \right) > \alpha_{\gamma - 1}.
  \end{equation}
  However, because $t$ is invertible, the image of $t^{-1}$ must cover all possible $z$ values. The quantity
  $
  \exp\left( \frac{1}{2} (2 \Delta_{\gamma}^{\T}x + \Vert \Delta_{\gamma} \Vert^2) \right)
  $
  ranges over $(0, \infty)$ as $x$ varies over $\mathbb{R}^d$, so \cref{eq:deterministic-cond} cannot hold for every $x = t^{-1}(z)$ and every $\alpha_{\gamma - 1} \in (0,1]$. This contradiction proves that no deterministic residual correction satisfying (i) and (ii) exists.
\end{proof}

\section{Block Verification Complexity}%
\label{app:block-complexity}

\begin{proposition}%
  \label{thm:complexity-app}
  Assume the use of the Frozen Drafter and $\Tr(\cov)[q_{\textup{data}}] \leq \beta d$.
  For a fixed draft size $\gamma$, let $\rho(\gamma) \in [0,1]$ denote the ratio of the expected number of rounds required to complete speculative diffusion with block verification relative to sample verification.
  Taking $\gamma \asymp (\nicefrac{K}{\beta \delta d})^{1/3}$, the expected number of parallel target model calls under block verification is at most $\bigoh({\rho}(\gamma) K^{2/3} (\beta d \delta)^{1/3})$.
\end{proposition}
\begin{proof}
  The result follows from a combination of prior results.
  We will first establish that the number of rounds to complete speculative diffusion is smaller when utilizing block verification rather than sample verification. Through this connection between sample and block verification, we utilize a prior result that directly bounds the number of rounds required to complete speculative diffusion with sample verification.
  Finally, we note that the expected number of primary calls is directly proportional to the number of rounds of speculative diffusion---specifically 2, one call for the Frozen Drafter and one (parallel) call for the verification.

  Firstly, from \citet[Theorem 2]{sun2025block} we know that the expected number of decoded samples after $i$ many rounds of block verification speculative decoding is greater than its corresponding sample algorithm.
  Let $\K_\token(i)$ and $\K_\textrm{block}(i)$ be random variables that correspond to the number of denoising steps completed by speculative diffusion with sample and block verification, respectively. From \citet[Lemma 5]{sun2025block}, we know that for any $l \geq 0$
  \begin{equation}
    \label{eq:decoding-token-block-inequality}
    \Pr[\K_\token(i) \geq l]
    \leq
    \Pr[\K_\textrm{block}(i) \geq l].
  \end{equation}
  This result follows by marginalizing \citet[Lemma 5]{sun2025block} over the context and output.

  We will connect this result to the number of rounds of speculative diffusion required to be completed to complete the full denoising of $K$ fixed steps.
  The number of rounds required for denoising will be denoted by the random variable $\R_\token$ and $\R_\textrm{block}$.
  It follows that we can connect the random variable associated with $\K_\cdot$ and $\R_\cdot$ by the following events:
  \begin{equation}
    \R_{a} > i \iff \K_a(i) < K,
  \end{equation}
  where $a \in \{\token, \textrm{block}\}$.
  As a result, we have $\Pr[\R_{a} > i] = \Pr[\K_a(i) < K]$.
  Note that $\R_a$ is clearly a non-negative integer valued random variable. As a result, the expected number of rounds to denoise will be determined by a sum of tail probabilities. This gives us
  \begin{align*}
    \expect[\R_\token]
    &= \sum_{i=0}^{\infty} \Pr[\R_\token > i] \\
    &= \sum_{i=0}^{\infty} \Pr[\K_\token(i) < K] \\
    &= \sum_{i=0}^{\infty} (1 - \Pr[\K_\token(i) \geq K]) \\
    &\geq \sum_{i=0}^{\infty} (1 - \Pr[\K_\textrm{block}(i) \geq K]) \\
    &= \sum_{i=0}^{\infty} \Pr[\K_\textrm{block}(i) < K] \\
    &= \sum_{i=0}^{\infty} \Pr[\R_\textrm{block} > i]
    =  \expect[\R_\textrm{block}],
  \end{align*}
  where the inequality directly comes from \cref{eq:decoding-token-block-inequality}.

  Hence, for a fixed draft size $\gamma$, there exists a $\rho(\gamma) \in [0, 1]$ such that
  \begin{equation}
    \label{eq:relate-R-token-block}
    \rho(\gamma) \expect[\R_\token]
    =
    \expect[\R_\textrm{block}].
  \end{equation}
  Now with our given assumptions that $\Tr(\cov[q_{\textup{data}}]) \leq \beta d$, from the proof of \citet[Theorem 18]{hu2025diffusion}, we have
  \begin{equation}
    \expect[\R_\token] \lesssim \frac{K}{\gamma} + \sqrt{K\gamma \delta \beta d}.
  \end{equation}
  With \cref{eq:relate-R-token-block}, this immediately gives us
  \begin{equation}
    \label{eq:block-round-bound}
    \expect[\R_\textrm{block}] \lesssim \rho(\gamma) \left( \frac{K}{\gamma} + \sqrt{K\gamma \delta \beta d} \right).
  \end{equation}
  Finally, taking $\gamma = \lceil (\nicefrac{K}{\delta \beta d})^{1/3} \rceil$ yields the desired result.
\end{proof}

\begin{remark}
  It should be noted that \cref{thm:complexity} simply converts the rate for sample verification into its block verification counterpart. However, it does not directly account for circumstances where the improvement in acceptance---quantified by $\rho(\gamma)$---can be extremely advantageous ($\rho(\gamma)$ small) for certain $\gamma$. 
  Indeed, one can express the expected number of rounds for block verification while optimizing over the draft size $\gamma$. From \cref{eq:block-round-bound} we have
  \begin{equation}
    \label{eq:block-round-explicit-gamma}
    \min_\gamma \, \expect[\R^\gamma_\textrm{block}] \lesssim \left( K^{2/3} (\beta d \delta)^{1/3} \right) \min_\gamma \left\{ \rho(\gamma) \left( \frac{1}{\gamma} \left( \frac{K}{\beta d \delta} \right)^{1/3} + \sqrt{\gamma\left( \frac{\beta d \delta} {K}\right)^{1/3}} \right) \right\},
  \end{equation}
  where we make the dependence of $\gamma$ for $\R^\gamma_\textrm{block}$ explicit.

  It becomes apparent from \cref{eq:block-round-explicit-gamma} that the choice of $\gamma = \lceil (\nicefrac{K}{\delta \beta d})^{1/3} \rceil$ need not be optimal, and if $\rho(\gamma)$ has certain structure, block verification yields a more advantageous rate than \cref{thm:complexity} for an alternative draft size.
\end{remark}

\section{Frozen and Free Drafters}
\label{app:drafters}

In the following section, we prove \cref{prop:speedup}.
We begin by formally defining the block efficiency of speculative diffusion. Let $\be(k)$ denote the average number of samples returned when the input of a round of speculative diffusion $y_{k}$ corresponds to the $k$th denoising step. Note that $\be(k) \in [1, \gamma + 1]$---at worst we denoise a single (corrected) step and at best we accept the full draft and denoise an extra step with the primary model. We note that $\be$ depends on the drafter $p$, the primary model $q$, and the exact acceptance/verification choice of the speculative diffusion algorithm. For the purpose of this section, we will leave the dependence of these on $\be(k)$ implicit.

To calculate the speedup of speculative diffusion using different drafters (in idealized conditions), we consider a modification of the following formulation explored in \citet{huang2025moesd}:
\begin{equation}
  \label{eq:speedup}
  \texttt{Speedup}(k) = \be(k) \left( \frac{\texttt{Time}(p, \gamma)}{\texttt{Time}(q, 1)} + \frac{\texttt{PTime}(q, \gamma + 1)}{\texttt{Time}(q, 1)} + \frac{\texttt{Time}(r, 1)}{\texttt{Time}(q, 1)} \right)^{-1},
\end{equation}
where $\texttt{Time}(\cdot, c)$ corresponds to the wall-clock time of evaluating the drafter, primary model, or residual calculation sequentially for $c$ values. $\texttt{PTime}$ corresponds to the parallel time.

For LLM speculative decoding, the drafter typically scales linearly in wall-clock time \wrt its return size, \ie, $\texttt{Time}(p, \gamma) = \gamma \texttt{Time}(p, 1)$.

In terms of dependencies, only the first term in \cref{eq:speedup} affects the speedup of speculative diffusion. The other terms can be considered ``fixed'' in our analysis of drafters.
One typically wishes to choose a drafter $p$ which maximizes $\be(k)$ while minimizing $\texttt{Time}(p, \gamma)$.

In an ideal setting, one can assume that the wall-clock time for computing the residual is negligible $\texttt{Time}(r, 1) \approx 0$. Furthermore, we can assume that the wall-clock time for computing $\gamma + 1$ parallel calls to the primary model is roughly equivalent to a single sequential primary model call ${\texttt{PTime}(q, \gamma + 1)} \approx {\texttt{Time}(q, 1)}$. When we state ``ideal conditions'', we assume that:
\begin{equation}
  \epsilon = \frac{\texttt{PTime}(q, \gamma + 1) - \texttt{Time}(q, 1)}{\texttt{Time}(q, 1)} + \frac{\texttt{Time}(r, 1)}{\texttt{Time}(q, 1)},
\end{equation}
for some small constant $\epsilon > 0$.

\begin{remark}
  The conditions are close to ideal whenever the cost of the primary model dominates all other times and ${\textup{\texttt{PTime}}(q, \gamma + 1)} \approx {\textup{\texttt{Time}}(q, 1)}$.
\end{remark}

\begin{proposition}%
  \label{prop:speedup-app}
  Let $\be(k)$ denote the block efficiency of a speculative diffusions algorithm with the Frozen Drafter at the $k$th denoising step. The speedup compared to no speculation  in ideal conditions is equal to $\be(k) / (2 + \epsilon)$.
\end{proposition}
\begin{proof}
  In ideal conditions, the cost of the Frozen Drafter for an entire draft proposal of size $\gamma$ is the single sequential call of the primary model. Thus we have $\texttt{Time}(p, \gamma) = {\texttt{Time}(q, 1)}$. Taking the other ideal conditions into account, \cref{eq:speedup} simplifies to $\be(k) / (2 + \epsilon)$, as required.
\end{proof}

The Free Drafter does not have the same $\approx 1/2$ scaling in ideal conditions.

\begin{proposition}%
  \label{prop:speedup-free}
  Let $\be(k)$ denote the block efficiency of a speculative diffusion algorithm with the Free Drafter at the $k$th denoising step, for $k \geq 2$. The speedup compared to no speculative diffusion in ideal conditions is equal to $\be(k) / (1 + \epsilon' + \epsilon)$, where $\epsilon' > 0$ is the cost of computing Free Drafter, \ie, $\epsilon' = \textup{\texttt{Time}}(p, \gamma) / \textup{\texttt{Time}}(q, 1)$.
\end{proposition}
\begin{proof}
  In ideal conditions, the cost of Free Drafter is free for an entire draft proposal (the cost of the primary model has been absorbed by the previous round of speculative diffusion). Thus $\texttt{Time}(p, \gamma) / \texttt{Time}(q, 1) = \epsilon'$. Thus the speedup will be simplified to $\be(k) / (1 + \epsilon + \epsilon')$.
\end{proof}

This provides a simple corollary.

\begin{corollary}%
  \label{cor:drafter-compare}
  Under ideal conditions, speculative diffusion with the Frozen Drafter is faster than speculative diffusion with the Free Drafter if
  \begin{equation}
    \frac{\be_\textup{frozen}(k)}{\be_{\textup{free}}(k)} > \frac{2 + \epsilon}{1 + \epsilon + \epsilon'}.
  \end{equation}
\end{corollary}

From \cref{cor:drafter-compare}, we can see that a necessary condition for the Frozen Drafter to be faster is that its block efficiency is close to twice as high as Free Drafter's counterpart.

\begin{remark}
  It should however be noted that when using the Frozen Drafter, one will at least have a block efficiency of $\be_\textup{frozen}(k) \geq 2$.  This is because the first draft image generated by the drafter will be exactly sampled by the primary model $q$ (this is our single primary model call) and thus we will always accept it.
\end{remark}

Experimentally, we show that the Free Drafter in practice can yield much better speedups than the Frozen Drafter. One reason---characteristic to speculative diffusion---is that the difficulty of accepting a draft image becomes progressively harder as we get closer to a clean image.
Take sample verification for example, where the per-sample acceptance probability is given by one minus the total-variation between the draft and primary model. Indeed, from \citet[Proposition 3.1]{bortoli2025accelerated}, the acceptance probability for sample verification speculative diffusions will be
\begin{equation}
  \textrm{TokenAcceptance}(k+1)
  =
  2 - 2\Phi\left( \frac{1}{2} \mathcal{M}(m^q_{k+1}, m^p_{k+1}; \sigma_{k+1}^2 \Id ) \right),
\end{equation}
where $\mathcal{M}^2(a, b; \Sigma) = (a - b)^\top \Sigma^{-1} (a - b)$ corresponds to the Mahalanobis distance with covariance $\Sigma$ and we remind that $\Phi$ is the \cdf of the standard normal distribution.

As one can see, as $\sigma_k \rightarrow 0$, the Mahalanobis distance will approach $\infty$. This causes the acceptance rate to go to $0$. Intuitively, when less noise is injected for each DDPM step, it becomes more difficult to accept a draft image. The block acceptance rate has a similar characteristic ($\alpha_k \rightarrow 0$).
This is independent of the form of the drafter.

\cref{fig:mahalanobis-acceptance} presents the one-step lagged acceptance probability for sample verification on ImageNet over 250 denoising steps and churn $\varepsilon = 0.25$.
This is equivalent to the second draft image acceptance probability of the Frozen Drafter or the first draft image acceptance probability of Free Drafter when the previous verification is sampled from the primary model (\ie, a previous full draft). As one can see, the acceptance probability is high but eventually becomes $0$ once we approach the end of denoising.
In such a case, we are mostly rejecting, thus the larger cost of the Frozen Drafter will heavily affect the speed of speculative diffusion, without its benefit of increased acceptance rate.

\begin{figure}
  \centering
  \includegraphics[width=1.0\linewidth]{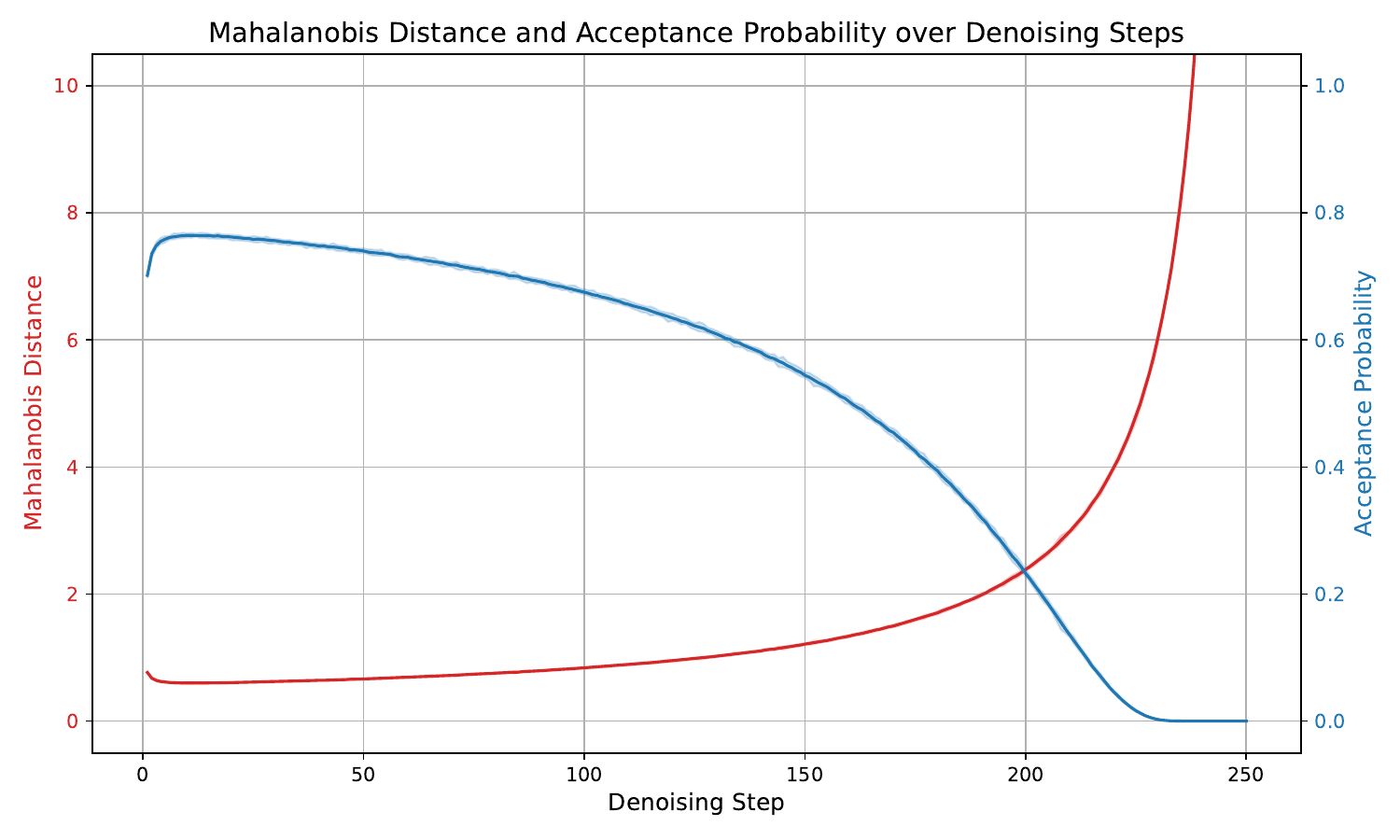}%
  \caption{Sample verification acceptance probability on the ImageNet LDM dataset. This plot uses a sampling configuration with $250$ denoising steps and churn parameter of $\varepsilon = 0.25$. We use 32 DDPM sampling trajectories. The one-step lagged predictions are then compared to the true model predictions to determine the Mahalanobis distance and acceptance probability.}%
  \label{fig:mahalanobis-acceptance}
\end{figure}

\section{Temperature}%
\label{app:temperature}

As a form of lenience for speculative diffusions, one can consider applying a temperature term to sample and block verification~\citep{bortoli2025accelerated}.
The primary motivation of utilizing a temperature term is to increase the acceptance rate of the draft sequence, while gracefully degrading quality (via a deviation from the primary model $q$).
One of the original proposals of speculative decoding~\citep{leviathan2023fast} suggested a form of lenience by replacing the density ratio ``$q / p$'' in verification with a scaled ratio ``$l q / p$'', where an increase in acceptance occurs with $l > 1$. Verification of this form can be shown to be part of an optimal relaxation (in terms of matching $q$) when the target distribution is slightly changed~\citep{Tran-Thien_2023}.
It should be noted that this optimal form of lossy speculative decoding could not be previously utilized as the residual deviates from what the reflection coupling calculates (\cref{algo:reflect-res}). However, the decomposition residual (\cref{algo:decomposition-res}) naturally allows for the required target distribution to implement lenience.
Recently, other approaches have also looked into replacing the verification procedure in speculative decoding to increase acceptance rate~\citep{zhong2025speeding,bachmann2025judge}.

In the following section, we study the unique temperature lenience available in speculative diffusion, as introduced in \citet{bortoli2025accelerated}.
Instead of linearly scaling the density ratio ``$q / p$'' we directly alter $p, q$ by changing their identity covariance matrices to a scaled version $\Id \mapsto \omega^2 \Id$.
For sample verification, we consider a version of \citet[Proposition F.1]{bortoli2025accelerated} specific to the decomposition residual with sample verification, as shown in \cref{algo:token-verification-temperature}. We also discuss and present a variant of block verification with temperature, as shown in \cref{algo:block-verification-temperature}.

We note that when we consider $\omega \neq 1$, the output distribution of speculative diffusion (of any kind) will not sample from the original target diffusion---the exactness guarantees in the main text only apply to the untempered $\omega = 1$ case.

\begin{algorithm}
  \begin{algorithmic}[1]
    \caption{$\texttt{SampleVerificationWithTemperature}( (\hat{y}_{k+i}, m^p_{k+i-1}, m^q_{k+i-1}, \sigma_{k+i-1})_{i=1}^\gamma; \omega)$}\label{algo:token-verification-temperature}
    \Require Draft $\hat{y}_{k+1:k+\gamma}$, draft means $m_{k:k+\gamma-1}^p$, target means $m_{k:k+\gamma-1}^q$, variances $\sigma^2_{k:k+\gamma-1}$, temperature $\omega > 0$.
    \State Set $\bm{c}_0 = 1$.
    \For{$j \in \{ 1, \ldots, \gamma \}$}
    \State Set $\Delta_{j} = (m_{k+j-1}^p - m_{k+j-1}^q) / \sigma_{k+j-1}$ and $\Z_{j} = (\hat{y}_{k+j} - m_{k+j-1}^p) / \sigma_{k+j-1}$.
    \State Calculate $\alpha_j = \min\left\{1, \frac{\mathcal{N}(\Z_j + \Delta_j; 0, \omega^2 \Id)}{\mathcal{N}(\Z_j; 0, \omega^2 \Id)} \right\}$.
    \State Flip coin $\bm{c}_j = \llbracket \eta < \alpha_j \rrbracket$, where $\eta \sim \Unif[0, 1]$.
    \EndFor
    \State Set $\tau = \min(\left\{ j \mid \bm{c}_j = 0 \right\} \cup \{ \gamma + 1 \})$.
    \State \Return $\tau$, $\Delta = \Delta_{\tau}$, $\Z = \Z_{\tau}$, $m^q = m^q_{k+\tau-1}$, $\sigma = \sigma_{k+\tau-1}$, $\alpha_{\textrm{res}} = 1$.\footref{foot:undefined}
  \end{algorithmic}
\end{algorithm}

\subsection{Sample Decomposition}

In the following, we explore using a temperature term in the verification step of sample-decomposition speculative diffusions. We recall that the output distribution of this setting corresponds to
\begin{equation}
  \pi_\token(y) = p(y) \alpha(y) + r(y) \left( 1 - \int p(\bar{y}) \alpha(\bar{y}) \, \dd \bar{y} \right),
\end{equation}
where
\begin{align*}
  p(y) &= \mathcal{N}(y; m^p, \sigma^2 \Id) \\
  \alpha(y) &= \min\left\{ 1, \frac{\mathcal{N}(y; m^q, \sigma^2 \Id)}{\mathcal{N}(y; m^p, \sigma^2 \Id)} \right\} \\
  r(y) &\propto \max\{ 0, \mathcal{N}(y; m^q, \sigma^2 \Id) - \mathcal{N}(y; m^p, \sigma^2 \Id) \}.
\end{align*}
Here no temperature has been added.

To add temperature, we modify the acceptance coin $c$ via
\begin{align*}
  \alpha^\omega(y) &= \min\left\{ 1, \frac{\mathcal{N}(y; m^q, \omega^2 \sigma^2 \Id)}{\mathcal{N}(y; m^p, \omega^2 \sigma^2 \Id)} \right\}.
\end{align*}
One of course has $\alpha^{\omega = 1} = \alpha$.

Now we can calculate
\begin{align*}
  &\pi^\omega_\token(y)
  = \mathcal{N}(y; m^p, \sigma^2 \Id) \min\left\{1, \frac{\mathcal{N}(y; m^q, \omega^2 \sigma^2 \Id)}{\mathcal{N}(y; m^p, \omega^2 \sigma^2 \Id)} \right\} \\
  & \quad + \norm ( \max\{ 0, \mathcal{N}(y; m^q, \sigma^2 \Id) - \mathcal{N}(y; m^p, \sigma^2 \Id) \} ) \left(1 - P_\textup{accept}(\omega) \right),
\end{align*}
where $\norm(f(x)) \defeq f(x) / \int f(x) \, \dd x$ unit normalizes a function and
\begin{equation}
  P_\textup{accept}(\omega) = \int p(\bar{y}) \alpha^\omega({\bar{y}}) \, \dd \bar{y}.
\end{equation}

Through a change of variables $y = m^q + \sigma z$, we have that the expected value of a function $f$ is given by,
\begin{align*}
  & \expect_{\pi^\omega_\token}[f(\Y^\omega)] \\
  &= \int f(m^q + \sigma z) \mathcal{N}(z - \Delta; 0, \Id) \min\left\{1, \frac{\mathcal{N}(z; 0, \omega^2 \Id)}{\mathcal{N}(z - \Delta; 0, \omega^2 \Id)} \right\} \\
  &\qquad + f(m^q + \sigma z) (1 - P_\textup{accept}(\omega)) \norm( \max\{ 0, \mathcal{N}(z; 0, \Id) - \mathcal{N}(z - \Delta; 0, \Id) \}  ) \, \dd z  \\
  &= \int f(m^q + \sigma z) \pi^\omega_\token(z) \, \dd z,
\end{align*}
where with abuse of notation we define $\pi^\omega_\token(z)$ (and each corresponding sub-function) as a function of $z$. One can check that $\pi^\omega_\token(z)$ is a valid distribution over noise values $z$ and that $P_\textup{accept}(\omega)$ can also be rewritten in terms of $z$.

As a result we have the following result.

\begin{proposition}
  Let $\Y^\omega$ be the output of
  speculative diffusion \cref{algo:abstract-spec-diffusion}
  with sample verification with temperature $\omega > 0$ \cref{algo:token-verification-temperature} and decomposition residual (\cref{algo:decomposition-res}, $\alpha = 1$). Then we have that
  \begin{equation}
    \expect_{\pi^\omega_\token}[\Y^\omega]
    = m^q +\sigma e g(\Vert \Delta \Vert, \omega),
  \end{equation}
  where $e = \Delta / \Vert \Delta \Vert$ and $\Delta = (m^p - m^q) / \sigma$.

  Furthermore, we have $g(\Vert \Delta \Vert, \omega) = 0$ for $\omega = 1$ and is an increasing function of $\omega$. Additionally for $\omega > 1$, $g(\Vert \Delta \Vert, \omega) \geq 0$; otherwise when $\omega < 1$, $g(\Vert \Delta \Vert, \omega) \leq 0$.
\end{proposition}
\begin{proof}
  One immediately has that
  \begin{equation*}
    \expect_{\pi^\omega_\token}[\Y^\omega]
    = m^q + \sigma \int z \pi^\omega_\token(z) \, \dd z,
  \end{equation*}
  which gives us the functional form of utilizing a temperature in verification. What remains is the explicit analysis of the latter term.

  To do so, we will consider a decomposition that involves the average tempered acceptance. Let $P_\textup{accept}(\omega)$ denote the average (\wrt the drafter distribution) rate of acceptance, \ie,
  \begin{equation}
    P_\textup{accept}(\omega) = \int p(z) \alpha^\omega(z) \, \dd z.
  \end{equation}

  Let
  \begin{align*}
    E(\omega) &= \expect_{\pi^\omega_\token(z)}[\Z]
    =
    \int z \pi^\omega_\token(z) \, \dd z, %
  \end{align*}
  where we note that $P_\textup{accept}(\omega)$ normalizes $p(z) \alpha^\omega(z)$ to be a probability measure.

  As $\omega = 1$ corresponds to the untempered case, we know that $E(1) = 0$. Then, from the fundamental theorem of calculus, we have that
  \begin{align*}
    E(\omega) = E(\omega) - E(1) = \int_{1}^{\omega} \frac{\dd E(t)}{\dd t} \, \dd t,
  \end{align*}
  Noting that our terms are continuous functions and (at least) piece-wise differentiable. We will examine the sign of the derivative to determine the sign of $E(\omega)$.
  Notice that
  \begin{align*}
    \frac{\dd E(\omega)}{\dd \omega}
    =
    \int z p(z) \left( \frac{\partial}{\partial \omega} \alpha^\omega(z) \right) \, \dd z
    -
    \mu_\textup{res} \frac{\partial P_\textup{accept}(\omega)}{\partial \omega}.
  \end{align*}

  First, we give a simplification of $\alpha^\omega(z)$.
  \begin{align*}
    \alpha^\omega(z)
    &= \min\left\{ 1, \frac{\mathcal{N}(z; 0, \omega^2 \Id)}{\mathcal{N}(z - \Delta; 0, \omega^2 \Id)} \right\} \\
    &= \min\left\{ 1, \exp\left( \frac{1}{2\omega^2} \left( \Vert \Delta \Vert^2 - 2 z^\top \Delta \right) \right) \right\}.
  \end{align*}
  Furthermore, the derivative of $\alpha^\omega(z)$ is non-zero only for values $ 2 z^\top \Delta > \Vert \Delta \Vert^2 $, \ie, when $\alpha^\omega(z) < 1$.
  In this case we have,
  \begin{align*}
    \frac{\partial \alpha^\omega(z)}{\partial \omega}
    &= \frac{\partial}{\partial \omega} \left( \frac{\mathcal{N}(z; 0, \omega^2 \Id)}{\mathcal{N}(z - \Delta; 0, \omega^2 \Id)} \right) \\
    &= \frac{\partial}{\partial \omega} \left( \exp\left( \frac{1}{2\omega^2} \left( \Vert \Delta \Vert^2 - 2 z^\top \Delta \right) \right) \right) \\
    &= \exp\left( \frac{1}{2\omega^2} \left( \Vert \Delta \Vert^2 - 2 z^\top \Delta \right) \right) \frac{2z^\top \Delta - \Vert \Delta \Vert^2}{\omega^3}.
  \end{align*}
  Thus, recombining to give our integral, we have
  \begin{align*}
    &\int z p(z) \left( \frac{\partial}{\partial \omega} \alpha^\omega(z) \right) \, \dd z \\
    & \quad =
    \int_{z \colon 2 z^\top \Delta > \Vert \Delta \Vert^2} z \mathcal{N}(z-\Delta; 0, \Id) \exp\left( \frac{1}{2\omega^2} \left( \Vert \Delta \Vert^2 - 2 z^\top \Delta \right) \right) \frac{2z^\top \Delta - \Vert \Delta \Vert^2}{\omega^3} \, \dd z.
  \end{align*}
  To evaluate this, we consider an orthogonal decomposition of $z = (e^\top z) e + z_\perp$, where we remind that $ e = \Delta / \Vert \Delta \Vert$. We will use a change of coordinates onto a basis corresponding to this decomposition (\ie, a basis where the first coordinate is along $e$). In particular, we utilize the reflection matrix $\tilde{P}_\refl = \Id - 2vv^\top / (v^\top v)$ where $v = \xi_1 - e$ and $\xi_1 = (1, 0, \ldots, 0)$ is the first standard basis vector. Here we note that $\tilde{P}_\refl \xi_1 = e$
  and
  \begin{equation*}
    \det(\tilde{P}_\refl)^2 = \det(\tilde{P}_\refl \tilde{P}_\refl) = \det( \Id ) = 1 \quad \implies \quad \vert \det(\tilde{P}_\refl) \vert = 1.
  \end{equation*}
  Thus through a change of coordinates via $\tilde{P}_\refl$ we have
  \begin{align*}
    &\frac{1}{\Vert \Delta \Vert} \int z p(z) \left( \frac{\partial}{\partial \omega} \alpha^\omega(z) \right) \, \dd z  \\
    &=
    \int \int_{u \colon 2 u > \Vert \Delta \Vert} (ue + z_{\perp}) \mathcal{N}(ue + z_{\perp}; \Delta, \Id) \exp\left( \frac{\Vert \Delta \Vert}{2\omega^2} \left( \Vert \Delta \Vert - 2 u \right) \right) \frac{2u - \Vert \Delta \Vert}{\omega^3} \, \dd u  \, \dd z_\perp \\
    &=
    \int \int_{\Vert \Delta \Vert / 2}^{\infty} (ue + z_{\perp}) \mathcal{N}(ue + z_{\perp}; \Delta, \Id) \exp\left( \frac{\Vert \Delta \Vert}{2\omega^2} \left( \Vert \Delta \Vert - 2 u \right) \right) \frac{2u - \Vert \Delta \Vert}{\omega^3} \, \dd u \, \dd z_\perp \\
    &=
    \int \int_{\Vert \Delta \Vert / 2}^{\infty} (ue + z_{\perp}) \mathcal{N}(ue + z_{\perp} - \Vert \Delta \Vert e; 0, \Id) \exp\left( \frac{\Vert \Delta \Vert}{2\omega^2} \left( \Vert \Delta \Vert - 2 u \right) \right) \frac{2u - \Vert \Delta \Vert}{\omega^3} \, \dd u \, \dd z_\perp \\
    &=
    \int \int_{\Vert \Delta \Vert / 2}^{\infty} (ue + z_{\perp}) \mathcal{N}((u - \Vert \Delta \Vert)e + z_{\perp}; 0, \Id) \exp\left( \frac{\Vert \Delta \Vert}{2\omega^2} \left( \Vert \Delta \Vert - 2 u \right) \right) \frac{2u - \Vert \Delta \Vert}{\omega^3} \, \dd u \, \dd z_\perp \\
    &=
    \int \int_{\Vert \Delta \Vert / 2}^{\infty} (ue + z_{\perp}) \mathcal{N}(u - \Vert \Delta \Vert; 0, 1)\mathcal{N}(z_{\perp}; 0, \Id) \exp\left( \frac{\Vert \Delta \Vert}{2\omega^2} \left( \Vert \Delta \Vert - 2 u \right) \right) \frac{2u - \Vert \Delta \Vert}{\omega^3} \, \dd u \, \dd z_\perp,
  \end{align*}
  where we are using the fact that $e^\top z_{\perp} = 0$ by construction.
  Now noting that the expectation of $\mathcal{N}(z_{\perp}; 0, \Id)$ is $0$ (and that it normalizes to $1$), we have
  \begin{align*}
    &\int z p(z) \left( \frac{\partial}{\partial \omega} \alpha^\omega(z) \right) \, \dd z  \\
    &=
    e \Vert \Delta \Vert \int_{\Vert \Delta \Vert / 2}^{\infty} u \mathcal{N}(u - \Vert \Delta \Vert; 0, 1)\exp\left( \frac{\Vert \Delta \Vert}{2\omega^2} \left( \Vert \Delta \Vert - 2 u \right) \right) \frac{2u - \Vert \Delta \Vert}{\omega^3} \, \dd u \\
    &\defeq e A(\omega).
  \end{align*}
  Notice, the integral term $A(\omega)$ is non-negative due to the bounds of integration. Thus the entire term's sign is dictated by $e = \Delta / \Vert \Delta \Vert$.

  Similarly, we can simplify $\partial_\omega P_\textup{accept}(\omega)$ via
  \begin{align*}
    & \frac{\partial P_\textup{accept}(\omega)}{\partial \omega} \\
    &=
    \int p(z) \left( \frac{\partial}{\partial \omega} \alpha^\omega(z) \right) \, \dd z  \\
    &= \Vert \Delta \Vert
    \int \int_{\Vert \Delta \Vert / 2}^{\infty} \mathcal{N}(u - \Vert \Delta \Vert; 0, 1)\mathcal{N}(z_{\perp}; 0, \Id) \exp\left( \frac{\Vert \Delta \Vert}{2\omega^2} \left( \Vert \Delta \Vert - 2 u \right) \right) \frac{2u - \Vert \Delta \Vert}{\omega^3} \, \dd u \, \dd z_\perp \\
    &= \Vert \Delta \Vert
    \int_{\Vert \Delta \Vert / 2}^{\infty} \mathcal{N}(u - \Vert \Delta \Vert; 0, 1) \exp\left( \frac{\Vert \Delta \Vert}{2\omega^2} \left( \Vert \Delta \Vert - 2 u \right) \right) \frac{2u - \Vert \Delta \Vert}{\omega^3} \, \dd u \\
    &\defeq B(\omega).
  \end{align*}
  Which again $B(\omega)$ is positive.

  Finally, we consider $\mu_\textup{res}$. We consider the integral of the un-normalized term:
  \begin{align*}
    \int z \max\{ 0, \mathcal{N}(z; 0, \Id) - \mathcal{N}(z - \Delta; 0, \Id) \} \, \dd z.
  \end{align*}
  The inner term is only positive when
  \begin{align*}
    \mathcal{N}(z; 0, \Id) - \mathcal{N}(z - \Delta; 0, \Id) > 0
    \iff \quad &
    \frac{\mathcal{N}(z; 0, \Id)}{\mathcal{N}(z - \Delta; 0, \Id)} > 1 \\
    \iff \quad &
    \exp\left( \frac{1}{2\omega^2} \left( \Vert \Delta \Vert^2 - 2 z^\top \Delta \right) \right) > 1 \\
    \iff \quad &
    \frac{1}{2\omega^2} \left( \Vert \Delta \Vert^2 - 2 z^\top \Delta \right) > 0 \\
    \iff \quad &
    2 z^\top e < \Vert \Delta \Vert.
  \end{align*}
  Notice, that the condition on $z^\top e$ is flipped this time.
  Now finally, we do the same change of coordinates as above to yield
  \begin{align*}
    & \int z \max\{ 0, \mathcal{N}(z; 0, \Id) - \mathcal{N}(z - \Delta; 0, \Id) \} \, \dd z \\
    &=
    \int \int_{-\infty}^{\Vert \Delta \Vert / 2}
    (ue + z_{\perp}) \left( \mathcal{N}(ue + z_{\perp}; 0, \Id) - \mathcal{N}((u - \Vert \Delta \Vert)e + z_{\perp}; 0, \Id) \right) \, \dd u \, \dd z_\perp \\
    &=
    \int \int_{-\infty}^{\Vert \Delta \Vert / 2}
    (ue + z_{\perp}) \mathcal{N}(z_{\perp}; 0, \Id) \left( \mathcal{N}(u; 0, 1) - \mathcal{N}(u - \Vert \Delta \Vert; 0, 1) \right) \, \dd u \, \dd z_\perp \\
    &=
    e \int_{-\infty}^{\Vert \Delta \Vert / 2}
    u \mathcal{N}(u; 0, 1) - u \mathcal{N}(u - \Vert \Delta \Vert; 0, 1) \, \dd u \\
    &=
    e \int_{-\infty}^{\Vert \Delta \Vert / 2}
    u \mathcal{N}(u; 0, 1) \, \dd u - e \int_{-\infty}^{- \Vert \Delta \Vert / 2} (v + \Vert \Delta \Vert) \mathcal{N}(v; 0, 1) \, \dd v \\
    &=
    e \int_{-\infty}^{\Vert \Delta \Vert / 2}
    u \mathcal{N}(u; 0, 1) \, \dd u - e \int_{-\infty}^{- \Vert \Delta \Vert / 2} v \mathcal{N}(v; 0, 1) \, \dd v - e \Vert \Delta \Vert \int_{-\infty}^{- \Vert \Delta \Vert / 2} \mathcal{N}(v; 0, 1) \, \dd v \\
    &= e \left( \left[ - \mathcal{N}(u; 0, 1)\bigg|^{u = \Vert \Delta \Vert / 2}_{u = - \infty} \right] - \left[ - \mathcal{N}(u; 0, 1)\bigg|^{u = -\Vert \Delta \Vert / 2}_{u = - \infty} \right] - \Vert \Delta \Vert \Phi\left( - \frac{\Vert \Delta \Vert}{2} \right) \right) \\
    &= e \left( \mathcal{N}\left( - \frac{\Vert \Delta \Vert}{2} ; 0, 1 \right) - \mathcal{N}\left( \frac{\Vert \Delta \Vert}{2} ; 0, 1 \right) - \Vert \Delta \Vert \Phi\left( - \frac{\Vert \Delta \Vert}{2} \right) \right) \\
    &= - e \Vert \Delta \Vert \Phi\left( - \frac{\Vert \Delta \Vert}{2} \right),
  \end{align*}
  where we are using the symmetry of the normal distribution \pdf.

  Thus together, we get
  \begin{align*}
    \mu_\textup{res}
    &= - e \Vert \Delta \Vert  \frac{\Phi\left( - \frac{\Vert \Delta \Vert}{2} \right)}{\int_{-\infty}^{\Vert \Delta \Vert / 2}
    \left( \mathcal{N}(u; 0, 1) - \mathcal{N}(u - \Vert \Delta \Vert; 0, 1) \right) \, \dd u} \\
    &= - e \Vert \Delta \Vert \frac{\Phi\left( - \frac{\Vert \Delta \Vert}{2} \right)}{\Phi\left( \frac{\Vert \Delta \Vert}{2} \right) - \Phi\left( - \frac{\Vert \Delta \Vert}{2} \right)} \\
    &= - e \Vert \Delta \Vert \frac{ 1 - \Phi\left( \frac{\Vert \Delta \Vert}{2} \right) }{2 \Phi\left( \frac{\Vert \Delta \Vert}{2} \right) - 1} \\
    &\defeq - e C.
  \end{align*}
  where the fractional term $C$ is non-negative.

  Hence we have that
  \begin{equation}
    \frac{\dd E(\omega)}{\dd \omega}
    = e \left( A(\omega) + B(\omega) C \right).
  \end{equation}
  And thus
  \begin{equation}
    E(\omega) = e \int_{1}^{\omega} A(t) + B(t) C  \, \dd t,
  \end{equation}
  with $A(t) + B(t) C \geq 0$  for all $t \geq 0$.
  Thus with $\omega > 1$, $E(\omega)$ increases in magnitude while following the sign of $e$. When $\omega < 1$, $E(\omega)$ increases in magnitude in the opposite direction/sign of $e$.
\end{proof}

\subsection{Block}

Adding temperature to block verification is not as straightforward as its sample verification counterpart. There are multiple ways of doing this, where the differences come from determining where the influence of the temperature $\omega$ starts and stops, \ie, should $\alpha_i$, which is passed to \cref{algo:decomposition-res}, be computed with temperature. For simplicity, we propose and experimentally test a variant of temperature using the verification of \cref{algo:block-verification-temperature}. In addition, when calculating the decomposition residual, we utilize the temperature scaled $\alpha_i$'s.
Again, we note that for $\omega \neq 1$, \cref{algo:block-verification-temperature} is lossy and does not guarantee that speculative diffusion recovers the target diffusion.

\begin{algorithm}
  \caption{$\texttt{BlockVerificationWithTemperature}((\hat{y}_{k+i}, m^p_{k+i-1}, m^q_{k+i-1}, \sigma_{k+i-1})_{i=1}^\gamma; \omega)$}\label{algo:block-verification-temperature}
  \begin{algorithmic}[1]
    \Require Draft $\hat{y}_{k+1:k+\gamma}$, draft means $m_{k:k+\gamma-1}^p$, target means $m_{k:k+\gamma-1}^q$, variances $\sigma^2_{k:k+\gamma-1}$, temperature $\omega > 0$.
    \State Initialize $\bm{c}_0 = 1$ and $\alpha_0 = 1$.
    \For{$j \in \{ 1, \ldots, \gamma \}$}
    \State Set $\Delta_{j} = (m_{k+j-1}^p - m_{k+j-1}^q) / \sigma_{k+j-1}$.
    \State Set $\Z_j = (\hat{y}_{k+j} - m_{k+j-1}^p) / \sigma_{k+j-1}$.
    \State Calculate $\alpha_j = \min\left\{1, \alpha_{j-1} \frac{\mathcal{N}(\Z_j + \Delta_{j}; 0, \omega^2 \Id)}{\mathcal{N}(\Z_j; 0, \omega^2 \Id)} \right\}$.\label{algo-line:block-coin-temperature}
    \EndFor
    \For{$j \in \{ 1, \ldots, \gamma \}$}
    \If{$j \neq \gamma$}
    \State Calculate $h_j = {v_j}/({v_j + 1 - \alpha_j})$, where
    \begin{equation}
      v_j = \alpha_j \Phi\left( \frac{\log \alpha_j}{\Vert \Delta_{j+1} \Vert} + \frac{\Vert \Delta_{j+1} \Vert}{2}\right) - \Phi\left( \frac{\log \alpha_j}{\Vert \Delta_{j+1} \Vert} - \frac{\Vert \Delta_{j+1} \Vert}{2}\right).
    \end{equation}
    \Else
    \State Set $h_\gamma = \alpha_\gamma$
    \EndIf
    \State Flip coin $\bm{c}_j = \llbracket \eta < h_j \rrbracket$, where $\eta \sim \Unif[0, 1]$.
    \EndFor
    \State Set $\tau = 1 + \max\left\{ j \in \{ 0, 1, \ldots, \gamma\} \mid \bm{c}_j = 1 \right\}$.
    \State \Return $\tau$, $\Delta = \Delta_{\tau}$, $\Z = \Z_{\tau}$, $m^q = m^q_{k+\tau-1}$, $\sigma = \sigma_{k+\tau-1}$, $\alpha_{\textrm{res}} = \alpha_{\tau-1}$.\footref{foot:undefined}
  \end{algorithmic}
\end{algorithm}

\section{Experimental details}
\label{app:experiments}

\paragraph{Pixel Experiments.}
We follow the experimental setup of \citet{bortoli2025accelerated}. For all pixel experiments, we consider a U-Net architecture~\citep{ronneberger2015u}. For each U-Net, the channel size is 192 for ImageNet and 256 for the other datasets. We also apply multipliers for the channels at each level, where the multipliers are $(1, 2, 3, 4)$ for ImageNet and $(1, 2, 2, 2)$ for other datasets. At each level, we consider 3 residual blocks for ImageNet and 2 residual blocks for other datasets. An attention layer is applied on the second level of the U-Net and attention is also used at the bottleneck block. Each residual block consists of the following: the input is normalized, a $3 \times 3$ convolution block is applied, non-linear activation is applied, a dropout layer, and then a final $3 \times 3$ convolution block is applied. A residual connection is also utilized to add the input via a $1 \times 1$ convolutional block. RMS normalization and GELU are used for the normalization layer and non-linearity, respectively. 
When applicable, multi-head attention with 8 heads is used after the convolutional residual block.
For a time embedding, we use a 192-dimensional sinusoidal embedding for ImageNet and 256-dimensional sinusoidal embedding for other datasets. When the dataset has labels, we embed the different classes and consider a conditional model, where we also condition on the augmentation vector. For training, we use the Adam optimizer with additional norm clipping (equal to 1) with learning rate $1e-5$.

\paragraph{Latent Experiments.}
Our latent models follow the recipe of \citet{rombach2022high}. We train an autoencoder to encode images of shape $256 \times 256 \times 3$ to latent tensors of shape $64 \times 64 \times 3$.
The autoencoder architectures utilized mirror LDM-4 for their respective datasets, where we use $\beta$-VAE with $\beta = 1e-6$ for both datasets.
For CelebA, we utilize a 10.0 coefficient on adversarial loss and 10.0 coefficient on generator loss.
For ImageNet, we utilize a 0.5 coefficient on adversarial loss and 1.0 coefficient on generator loss.
For the actual latent diffusion model, we also mirror LDM-4 in \citep{rombach2022high}.
For CelebA, we use a channel size of 224 with $(1, 2, 3, 4)$ multipliers. We utilize a dropout rate of 0.2 and 2 residual blocks. A learning rate of $9.6e-5$ is used with Adam (with norm clipping equal to 1).
For ImageNet, we use a channel size of 256 with $(1, 2, 3, 5)$ multipliers. We use a dropout of 0.1 and 2 residual blocks. A learning rate of $1e-4$ is used with the same optimizer.

\paragraph{Sampling.}
Both in sampling and training, we use the rectified flow noise schedule~\citep{liu2022flow}. We utilize the sampling method in \citet{ho2020denoising} with a safety epsilon of $1e-4$.
Speculative diffusions are then used to speed up this base sampling method.
In \cref{algo:u-sample}, the two while-loops suggest that we repeat the loop until a termination criterion is met. In practice, we find that fixing the number of iterations does not cause a major degradation in performance. For ease of implementation in \verb+JAX+, we establish the lower bound $u_\textrm{lower}$ by a maximum of $3$ iterations and run the binary search with a maximum of $4$ iterations.

\paragraph{Error Propagation.}
In some of our result tables, we give an estimate of the \std for the speedup \wrt the DDPM wall-clock time. Although we have the \std of the wall-clock time of the speculative diffusion sampling and the DDPM sampling, we do not have the \std of the ratio. We propagate the error via a first-order Taylor series approximation, see for instance \citet[Chapter 4]{lee2005analyzing}.

\section{Additional Experiments}

\paragraph{Additional FID scores}
We present additional results to verify speedups and FID scores. \cref{tab:fids_0_25,tab:fids_0_5} present these results over all datasets for churn parameters $\varepsilon = 0.25$ and $\varepsilon = 0.5$. This is fixed for a draft sequence size of $\gamma = 7$. FID scores for all datasets are calculated over 50k samples, except for CelebA LDM which utilizes 30k samples.

We note that the original sampling FID and speculative diffusion FID are similar across these results. The FID scores of the decomposition ($\decomposition$) and block ($\block$) approaches tend to be more similar in FID scores, which is expected as the residual distributions they target are similar.

\begin{table}[ht]
  \centering
  \caption{Wall-clock speedups and FID scores over all datasets and different number of denoising steps. The churn parameter is set to $\varepsilon = 0.25$ and the window size is set to $\gamma = 7$.
    Non-FID quantities are calculated over 500 samples and the $\pm$ error ranges are approximated via error propagation.
  }
  \label{tab:fids_0_25}
  \resizebox{\textwidth}{!}{
    \begin{tabular}{clcccccccc}
      \toprule
      & & \multicolumn{4}{c}{Wall-clock Speedup} & \multicolumn{4}{c}{FID} \\
      \cmidrule(lr){3-6}
      \cmidrule(lr){7-10}
      Dataset & Steps & \reflection & \decomposition & \block & \reflection$\uparrow$\block \% & \reflection & \decomposition & \block & DDPM \\
      \midrule
      \tabledataset{CelebA}{LDM} & 50 & $1.10 \pm 0.06$ & $1.10 \pm 0.06$ & $1.10 \pm 0.06$ & $-0.59\%$ & $6.17$ & $6.33$ & $6.36$ & $6.38$ \\
      & 100 & $1.48 \pm 0.07$ & $1.47 \pm 0.07$ & $1.48 \pm 0.08$ & $0.54\%$ & $6.19$ & $6.13$ & $6.18$ & $6.23$ \\
      & 250 & $2.20 \pm 0.09$ & $2.19 \pm 0.09$ & $2.25 \pm 0.10$ & $2.10\%$ & $6.42$ & $6.11$ & $6.13$ & $6.23$ \\
      & 500 & $2.94 \pm 0.11$ & $2.95 \pm 0.11$ & $3.04 \pm 0.11$ & $3.47\%$ & $6.43$ & $6.09$ & $6.08$ & $6.27$ \\
      & 1000 & $3.83 \pm 0.11$ & $3.83 \pm 0.10$ & $3.95 \pm 0.11$ & $3.27\%$ & $6.27$ & $6.12$ & $6.11$ & $6.22$ \\
      \midrule
      \tabledataset{CelebA}{Pixel} & 50 & $1.41 \pm 0.11$ & $1.40 \pm 0.12$ & $1.42 \pm 0.13$ & $0.67\%$ & $6.24$ & $6.28$ & $6.20$ & $6.24$ \\
      & 100 & $2.02 \pm 0.16$ & $2.03 \pm 0.17$ & $2.11 \pm 0.18$ & $4.52\%$ & $3.45$ & $3.46$ & $3.41$ & $3.44$ \\
      & 250 & $3.20 \pm 0.20$ & $3.18 \pm 0.21$ & $3.35 \pm 0.21$ & $4.81\%$ & $2.77$ & $2.61$ & $2.57$ & $2.63$ \\
      & 500 & $4.16 \pm 0.19$ & $4.16 \pm 0.19$ & $4.32 \pm 0.21$ & $3.79\%$ & $2.81$ & $2.62$ & $2.54$ & $2.63$ \\
      & 1000 & $5.08 \pm 0.17$ & $5.09 \pm 0.16$ & $5.22 \pm 0.16$ & $2.75\%$ & $2.73$ & $2.62$ & $2.60$ & $2.68$ \\
      \midrule
      \tabledataset{CIFAR10}{Pixel} & 50 & $1.52 \pm 0.19$ & $1.51 \pm 0.19$ & $1.51 \pm 0.21$ & $-0.90\%$ & $3.76$ & $3.77$ & $3.80$ & $3.85$ \\
      & 100 & $2.33 \pm 0.24$ & $2.34 \pm 0.25$ & $2.35 \pm 0.24$ & $1.10\%$ & $2.76$ & $2.75$ & $2.78$ & $2.85$ \\
      & 250 & $3.61 \pm 0.25$ & $3.59 \pm 0.24$ & $3.60 \pm 0.24$ & $-0.04\%$ & $2.13$ & $2.15$ & $2.14$ & $2.23$ \\
      & 500 & $4.49 \pm 0.21$ & $4.49 \pm 0.20$ & $4.45 \pm 0.20$ & $-0.90\%$ & $2.17$ & $2.21$ & $2.22$ & $2.29$ \\
      & 1000 & $5.22 \pm 0.15$ & $5.23 \pm 0.15$ & $5.12 \pm 0.14$ & $-2.00\%$ & $2.23$ & $2.26$ & $2.25$ & $2.37$ \\
      \midrule
      \tabledataset{ImageNet}{LDM} & 50 & $1.18 \pm 0.06$ & $1.18 \pm 0.07$ & $1.18 \pm 0.06$ & $0.35\%$ & $11.31$ & $10.97$ & $11.01$ & $11.22$ \\
      & 100 & $1.53 \pm 0.08$ & $1.53 \pm 0.08$ & $1.55 \pm 0.07$ & $0.98\%$ & $10.70$ & $10.13$ & $10.14$ & $10.77$ \\
      & 250 & $2.24 \pm 0.10$ & $2.23 \pm 0.10$ & $2.28 \pm 0.11$ & $1.76\%$ & $10.81$ & $9.64$ & $9.70$ & $10.52$ \\
      & 500 & $2.99 \pm 0.12$ & $2.99 \pm 0.13$ & $3.10 \pm 0.14$ & $3.80\%$ & $10.83$ & $9.58$ & $9.66$ & $10.16$ \\
      & 1000 & $3.88 \pm 0.12$ & $3.90 \pm 0.12$ & $4.05 \pm 0.13$ & $4.20\%$ & $10.59$ & $9.60$ & $9.59$ & $10.18$ \\
      \midrule
      \tabledataset{ImageNet}{Pixel} & 50 & $1.24 \pm 0.09$ & $1.24 \pm 0.10$ & $1.25 \pm 0.11$ & $1.20\%$ & $3.47$ & $3.42$ & $3.44$ & $3.50$ \\
      & 100 & $1.77 \pm 0.16$ & $1.77 \pm 0.16$ & $1.83 \pm 0.18$ & $3.47\%$ & $2.93$ & $2.97$ & $2.93$ & $2.97$ \\
      & 250 & $2.82 \pm 0.23$ & $2.81 \pm 0.22$ & $2.96 \pm 0.22$ & $4.70\%$ & $2.60$ & $2.59$ & $2.59$ & $2.66$ \\
      & 500 & $3.72 \pm 0.24$ & $3.71 \pm 0.23$ & $3.86 \pm 0.23$ & $3.75\%$ & $2.52$ & $2.49$ & $2.49$ & $2.57$ \\
      & 1000 & $4.58 \pm 0.20$ & $4.60 \pm 0.21$ & $4.76 \pm 0.19$ & $3.92\%$ & $2.47$ & $2.45$ & $2.45$ & $2.55$ \\
      \midrule
      \tabledataset{LSUN}{Pixel} & 50 & $1.18 \pm 0.10$ & $1.17 \pm 0.10$ & $1.17 \pm 0.10$ & $-0.43\%$ & $6.49$ & $6.56$ & $6.58$ & $6.46$ \\
      & 100 & $1.71 \pm 0.14$ & $1.70 \pm 0.14$ & $1.75 \pm 0.16$ & $2.42\%$ & $5.10$ & $5.15$ & $5.10$ & $5.08$ \\
      & 250 & $2.74 \pm 0.19$ & $2.74 \pm 0.17$ & $2.86 \pm 0.19$ & $4.44\%$ & $4.11$ & $4.05$ & $4.05$ & $4.03$ \\
      & 500 & $3.64 \pm 0.18$ & $3.63 \pm 0.18$ & $3.75 \pm 0.18$ & $3.03\%$ & $3.76$ & $3.64$ & $3.64$ & $3.70$ \\
      & 1000 & $4.48 \pm 0.16$ & $4.49 \pm 0.15$ & $4.59 \pm 0.15$ & $2.53\%$ & $3.67$ & $3.50$ & $3.51$ & $3.55$ \\
      \bottomrule
    \end{tabular}
  }
\end{table}

\begin{table}[ht]
  \centering
  \caption{Wall-clock speedups and FID scores over all datasets and different number of denoising steps. The churn parameter is set to $\varepsilon = 0.5$ and the window size is set to $\gamma = 7$.
    Non-FID quantities are calculated over 500 samples, and the $\pm$ error ranges are approximated via error propagation.%
  }%
  \label{tab:fids_0_5}
  \resizebox{\textwidth}{!}{
    \begin{tabular}{clcccccccc}
      \toprule
      & & \multicolumn{4}{c}{Wall-clock Speedup} & \multicolumn{4}{c}{FID} \\
      \cmidrule(lr){3-6}
      \cmidrule(lr){7-10}
      Dataset & Steps & \reflection & \decomposition & \block & \reflection$\uparrow$\block \% & \reflection & \decomposition & \block & DDPM \\
      \midrule
      \tabledataset{CelebA}{LDM} & 50 & $1.13 \pm 0.06$ & $1.13 \pm 0.07$ & $1.13 \pm 0.07$ & $-0.39\%$ & $6.34$ & $6.53$ & $6.62$ & $6.64$ \\
      & 100 & $1.47 \pm 0.07$ & $1.47 \pm 0.08$ & $1.49 \pm 0.07$ & $1.31\%$ & $6.31$ & $6.20$ & $6.31$ & $6.55$ \\
      & 250 & $2.15 \pm 0.09$ & $2.14 \pm 0.09$ & $2.20 \pm 0.10$ & $2.43\%$ & $6.43$ & $6.22$ & $6.27$ & $6.59$ \\
      & 500 & $2.84 \pm 0.10$ & $2.85 \pm 0.11$ & $2.94 \pm 0.11$ & $3.43\%$ & $6.46$ & $6.34$ & $6.32$ & $6.73$ \\
      & 1000 & $3.70 \pm 0.11$ & $3.72 \pm 0.10$ & $3.82 \pm 0.11$ & $3.07\%$ & $6.58$ & $6.62$ & $6.62$ & $6.75$ \\
      \midrule
      \tabledataset{CelebA}{Pixel} & 50 & $1.44 \pm 0.12$ & $1.43 \pm 0.12$ & $1.47 \pm 0.13$ & $1.60\%$ & $7.02$ & $7.18$ & $7.02$ & $6.99$ \\
      & 100 & $2.04 \pm 0.17$ & $2.04 \pm 0.16$ & $2.14 \pm 0.18$ & $4.95\%$ & $3.96$ & $3.90$ & $3.84$ & $3.91$ \\
      & 250 & $3.15 \pm 0.21$ & $3.14 \pm 0.20$ & $3.30 \pm 0.20$ & $4.83\%$ & $2.91$ & $2.77$ & $2.74$ & $2.86$ \\
      & 500 & $4.09 \pm 0.20$ & $4.09 \pm 0.19$ & $4.24 \pm 0.22$ & $3.90\%$ & $2.93$ & $2.62$ & $2.65$ & $2.83$ \\
      & 1000 & $5.03 \pm 0.19$ & $5.04 \pm 0.17$ & $5.16 \pm 0.17$ & $2.58\%$ & $2.87$ & $2.65$ & $2.67$ & $2.94$ \\
      \midrule
      \tabledataset{CIFAR10}{Pixel} & 50 & $1.65 \pm 0.21$ & $1.64 \pm 0.21$ & $1.67 \pm 0.22$ & $1.27\%$ & $3.91$ & $3.92$ & $3.91$ & $3.95$ \\
      & 100 & $2.48 \pm 0.26$ & $2.47 \pm 0.26$ & $2.50 \pm 0.26$ & $0.97\%$ & $2.91$ & $2.86$ & $2.85$ & $3.00$ \\
      & 250 & $3.65 \pm 0.24$ & $3.67 \pm 0.24$ & $3.67 \pm 0.24$ & $0.54\%$ & $2.19$ & $2.21$ & $2.24$ & $2.38$ \\
      & 500 & $4.51 \pm 0.21$ & $4.52 \pm 0.21$ & $4.45 \pm 0.18$ & $-1.34\%$ & $2.22$ & $2.21$ & $2.23$ & $2.52$ \\
      & 1000 & $5.23 \pm 0.15$ & $5.24 \pm 0.15$ & $5.13 \pm 0.14$ & $-1.91\%$ & $2.27$ & $2.30$ & $2.31$ & $2.73$ \\
      \midrule
      \tabledataset{ImageNet}{LDM} & 50 & $1.22 \pm 0.06$ & $1.22 \pm 0.06$ & $1.22 \pm 0.07$ & $0.09\%$ & $10.83$ & $10.77$ & $10.82$ & $10.78$ \\
      & 100 & $1.55 \pm 0.08$ & $1.55 \pm 0.08$ & $1.57 \pm 0.08$ & $1.47\%$ & $10.25$ & $10.06$ & $9.84$ & $10.28$ \\
      & 250 & $2.20 \pm 0.10$ & $2.20 \pm 0.10$ & $2.25 \pm 0.11$ & $2.32\%$ & $10.09$ & $9.55$ & $9.63$ & $9.81$ \\
      & 500 & $2.90 \pm 0.12$ & $2.91 \pm 0.13$ & $3.01 \pm 0.14$ & $3.82\%$ & $9.98$ & $9.66$ & $9.51$ & $9.64$ \\
      & 1000 & $3.78 \pm 0.13$ & $3.79 \pm 0.14$ & $3.91 \pm 0.14$ & $3.67\%$ & $9.90$ & $9.65$ & $9.41$ & $9.61$ \\
      \midrule
      \tabledataset{ImageNet}{Pixel} & 50 & $1.30 \pm 0.11$ & $1.29 \pm 0.11$ & $1.32 \pm 0.12$ & $1.68\%$ & $3.58$ & $3.62$ & $3.64$ & $3.66$ \\
      & 100 & $1.84 \pm 0.18$ & $1.81 \pm 0.17$ & $1.89 \pm 0.19$ & $3.05\%$ & $2.97$ & $2.99$ & $2.99$ & $3.04$ \\
      & 250 & $2.81 \pm 0.25$ & $2.80 \pm 0.24$ & $2.94 \pm 0.24$ & $4.66\%$ & $2.59$ & $2.64$ & $2.68$ & $2.73$ \\
      & 500 & $3.68 \pm 0.25$ & $3.69 \pm 0.24$ & $3.83 \pm 0.26$ & $4.12\%$ & $2.56$ & $2.50$ & $2.51$ & $2.65$ \\
      & 1000 & $4.57 \pm 0.23$ & $4.59 \pm 0.23$ & $4.75 \pm 0.23$ & $3.95\%$ & $2.51$ & $2.50$ & $2.50$ & $2.61$ \\
      \midrule
      \tabledataset{LSUN}{Pixel} & 50 & $1.21 \pm 0.11$ & $1.21 \pm 0.11$ & $1.21 \pm 0.11$ & $-0.13\%$ & $7.12$ & $7.18$ & $7.24$ & $7.01$ \\
      & 100 & $1.73 \pm 0.14$ & $1.73 \pm 0.14$ & $1.79 \pm 0.15$ & $3.39\%$ & $5.54$ & $5.52$ & $5.55$ & $5.46$ \\
      & 250 & $2.72 \pm 0.19$ & $2.72 \pm 0.18$ & $2.83 \pm 0.20$ & $4.08\%$ & $4.28$ & $4.26$ & $4.35$ & $4.46$ \\
      & 500 & $3.60 \pm 0.19$ & $3.61 \pm 0.18$ & $3.71 \pm 0.19$ & $2.87\%$ & $3.90$ & $3.79$ & $3.86$ & $4.20$ \\
      & 1000 & $4.50 \pm 0.17$ & $4.51 \pm 0.15$ & $4.59 \pm 0.17$ & $2.08\%$ & $3.83$ & $3.62$ & $3.62$ & $4.22$ \\
      \bottomrule
    \end{tabular}
  }
\end{table}

\paragraph{Temperature}
In the following section, we consider a simple experiment utilizing the temperature lenience schemes described in \cref{app:temperature}.
For a fixed churn of $\varepsilon = 0.25$ and draft size $\gamma = 7$,
\cref{tab:temperature-walltime} presents the wall-clock times of sampling and \cref{tab:temperature-fid} presents the FID score. An increase in $\omega$ corresponds to an increase in acceptance rate. $\omega = 1$ corresponds to the case without temperature. Notice that \wrt wall-clock time, an increase in $\omega$ (usually) decreases the wall-clock time. This is consistent for denoising steps $50, 100, 250$. However, for larger denoising steps $500, 1000$, there is an unexpected decrease in the wall-clock time for $\omega = 0.5$. This is unexpected as $\omega < 1$ should decrease acceptance rate, \ie, the verification value $\alpha_i$ will be smaller. A possible reason for this is that rejecting the drafter more at the start of the denoising process can allow for higher acceptance later on.

In contrast, the FID scores generally worsen as $\omega$ increases, consistent with accepting more draft proposals that would be rejected by the exact sampler. Despite this, the difference can be minimal when the number of denoising steps increases. A reason for this behavior is that as the number of denoising steps increases, the number of steps at the end of the denoising also increases. But even for a large $\omega > 1$, we will almost always be rejecting all draft sequences at the end of speculative diffusion. Thus, the final sampling from the primary model may allow for some ``correction'' of the increased initial draft sequence acceptance.
Nevertheless, in this specific experiment, it is clear that if one is utilizing a large number of denoising steps ($K=1000$), a higher temperature ($\omega = 2.0$) can achieve similar FID scores while also being almost twice as fast as its no temperature counterpart ($\omega = 1.0$).

\begin{table}[ht]
  \caption{Wall-clock (s) of speculative diffusion ($\varepsilon = 0.25$, $\gamma = 7$) for ImageNet LDM over different temperature values $\omega$.
    Quantities are calculated over 500 samples; the $\pm$ ranges are the empirical \std values.}%
  \label{tab:temperature-walltime}
  \resizebox{\textwidth}{!}{
    \centering
    \begin{tabular}{lccccccccc}
      \toprule
      & \multicolumn{3}{c}{$\omega=0.5$} & \multicolumn{3}{c}{$\omega=1.0$} & \multicolumn{3}{c}{$\omega=2.0$} \\
      \cmidrule(lr){2-4}
      \cmidrule(lr){5-7}
      \cmidrule(lr){8-10}
      Steps & \reflection & \decomposition & \block & \reflection & \decomposition & \block & \reflection & \decomposition & \block \\
      \midrule
      50 & $1.05 \pm 0.06$ & $1.05 \pm 0.06$ & $1.01 \pm 0.06$ & $1.01 \pm 0.05$ & $1.01 \pm 0.05$ & $1.01 \pm 0.05$ & $0.85 \pm 0.05$ & $0.85 \pm 0.05$ & $0.84 \pm 0.05$ \\
      100 & $1.64 \pm 0.08$ & $1.65 \pm 0.08$ & $1.55 \pm 0.08$ & $1.56 \pm 0.07$ & $1.56 \pm 0.07$ & $1.56 \pm 0.07$ & $1.20 \pm 0.06$ & $1.21 \pm 0.07$ & $1.18 \pm 0.06$ \\
      250 & $2.94 \pm 0.13$ & $2.96 \pm 0.13$ & $2.66 \pm 0.13$ & $2.85 \pm 0.10$ & $2.84 \pm 0.10$ & $2.82 \pm 0.10$ & $1.91 \pm 0.09$ & $1.90 \pm 0.09$ & $1.85 \pm 0.08$ \\
      500 & $4.59 \pm 0.17$ & $4.59 \pm 0.19$ & $3.97 \pm 0.18$ & $4.60 \pm 0.12$ & $4.62 \pm 0.13$ & $4.55 \pm 0.13$ & $2.81 \pm 0.10$ & $2.80 \pm 0.11$ & $2.72 \pm 0.10$ \\
      1000 & $7.17 \pm 0.24$ & $7.11 \pm 0.26$ & $6.12 \pm 0.22$ & $7.79 \pm 0.15$ & $7.78 \pm 0.16$ & $7.65 \pm 0.16$ & $4.42 \pm 0.11$ & $4.41 \pm 0.11$ & $4.30 \pm 0.09$ \\
      \bottomrule
    \end{tabular}
  }
\end{table}

\begin{table}[ht]
  \caption{FID 50k of speculative diffusion ($\varepsilon = 0.25$, $\gamma = 7$) for ImageNet LDM over different temperature values $\omega$.}%
  \label{tab:temperature-fid}
  \centering
  {%
    \begin{tabular}{lccccccccc}
      \toprule
      & \multicolumn{3}{c}{$\omega=0.5$} & \multicolumn{3}{c}{$\omega=1.0$} & \multicolumn{3}{c}{$\omega=2.0$} \\
      \cmidrule(lr){2-4}
      \cmidrule(lr){5-7}
      \cmidrule(lr){8-10}
      Steps & \reflection & \decomposition & \block & \reflection & \decomposition & \block & \reflection & \decomposition & \block \\
      \midrule
      50 & $10.54$ & $10.27$ & $10.36$ & $11.20$ & $11.07$ & $11.07$ & $13.60$ & $12.97$ & $13.11$ \\
      100 & $9.79$ & $9.22$ & $9.40$ & $10.67$ & $10.32$ & $10.33$ & $12.89$ & $11.54$ & $11.65$ \\
      250 & $10.07$ & $8.94$ & $9.02$ & $10.62$ & $10.00$ & $10.02$ & $11.74$ & $10.30$ & $10.29$ \\
      500 & $10.48$ & $8.97$ & $9.00$ & $10.69$ & $9.85$ & $9.89$ & $11.11$ & $9.84$ & $9.92$ \\
      1000 & $10.39$ & $9.23$ & $9.17$ & $10.53$ & $9.88$ & $9.84$ & $10.64$ & $9.71$ & $9.78$ \\
      \bottomrule
    \end{tabular}
  }
\end{table}

\end{document}